\documentclass[runningheads]{llncs}
\usepackage{graphicx}
\usepackage{comment}
\usepackage{amsmath,amssymb} 
\usepackage{color}
\usepackage{algorithmic}
\usepackage{algorithm2e}
\usepackage{microtype}
\usepackage{subfigure}
\usepackage{booktabs} % for professional tables

\usepackage{caption}
\usepackage{graphicx,xcolor} 

\DeclareMathOperator*{\argmin}{arg\,min}

\newcommand\numberthis{\addtocounter{equation}{1}\tag{\theequation}}
\newtheorem{thm}{Theorem}

\usepackage{hyperref}

\begin{document}
% \linenumbers
\pagestyle{headings}
\mainmatter
\def\ECCVSubNumber{4324}  % Insert your submission number here

\title{Variational Diffusion Autoencoders with Random Walk Sampling} % Replace with your title

\begin{comment}
\titlerunning{ECCV-20 submission ID \ECCVSubNumber} 
\authorrunning{ECCV-20 submission ID \ECCVSubNumber} 
\author{Anonymous ECCV submission}
\institute{Paper ID \ECCVSubNumber}
\end{comment}
%******************

%\begin{comment}
\titlerunning{Variational Diffusion Autoencoders with Random Walk Sampling}

\author{Henry Li\inst{1}\thanks{Equal contribution.} \and
Ofir Lindenbaum\inst{1}$^{\star}$ \and
Xiuyuan Cheng\inst{2} \and
Alexander Cloninger\inst{3}}
\authorrunning{H. Li, et al.}
\institute{Applied Mathematics, Yale University, New Haven, CT 06520, USA \email{\{henry.li,ofir.lindenbaum\}@yale.edu} \and
Department of Mathematics, Duke University, Durham, NC 27708, USA \email{xiuyuan.cheng@duke.edu} \and
Department of Mathematics and Halicio{\u g}lu Data Science Institute, University of California San Diego, La Jolla, CA 92093, USA\\
\email{acloninger@ucsd.edu}
}

%\end{comment}
%******************
\maketitle

\begin{abstract}
Variational autoencoders (VAEs) and generative adversarial networks (GANs) enjoy an intuitive connection to manifold learning: in training the decoder/generator is optimized to approximate a homeomorphism between the data distribution and the sampling space. This is a construction that strives to define the data manifold. A major obstacle to VAEs and GANs, however, is choosing a suitable prior that matches the data topology. Well-known consequences of poorly picked priors are posterior and mode collapse. To our knowledge, no existing method sidesteps this user choice.
Conversely, \textit{diffusion maps} automatically infer the data topology and enjoy a rigorous connection to manifold learning, but do not scale easily or provide the inverse homeomorphism (i.e. decoder/generator). We propose a method \footnote{\url{https://github.com/lihenryhfl/vdae}} that combines these approaches into a generative model that inherits the asymptotic guarantees of \textit{diffusion maps} while preserving the scalability of deep models. We prove approximation theoretic results for the dimension dependence of our proposed method. Finally, we demonstrate the effectiveness of our method with various real and synthetic datasets.
\keywords{deep learning, variational inference, manifold learning, image and video synthesis, generative models, unsupervised learning}
\end{abstract}

\section{Introduction}

Generative models such as variational autoencoders (VAEs, ~\cite{kingma2013auto}) and generative adversarial networks (GANs, \cite{goodfellow2014generative}) have made it possible to sample remarkably realistic points from complex high dimensional distributions at low computational cost. 
% VAEs jointly learn both a latent encoding and a generative model for data by combining variational inference with an auto-encoding architecture, while GANs use a game theoretic approach to simultaneously train a generator and an adversarial teacher. 
While the theoretical framework behind the two methods are different --- one is derived from variational inference and the other from game theory --- they both involve learning smooth mappings from a user-defined prior $p(z)$ to the data $p(x)$.

% The manifold hypothesis states that high dimensional data often lie on low dimensional manifolds.
When $p(z)$ is supported on a Euclidean space (e.g. $p(z)$ is Gaussian or uniform) and the $p(x)$ is supported on a manifold (i.e. the Manifold Hypothesis, see \cite{manifoldhypothesis1,manifoldhypothesis2}), VAEs and GANs become manifold learning methods, as manifolds themselves are defined as sets that are locally homeomorphic to Euclidean space. Thus the learning of such homeomorphisms may shed light on the success of VAEs and GANs in modeling complex distributions. 

This connection to manifold learning also offers a reason why these generative models fail --- when they do fail. Known as \textit{posterior collapse} in VAEs \cite{broken_elbo,info_vae,posterior_collapse_1,posterior_collapse_2} and \textit{mode collapse} in GANs \cite{goodfellow_mode_collapse}, both describe cases where the learned mapping collapses large parts of the input to a single point in the output. This violates the bijective requirement of a homeomorphism. It also results in degenerate latent spaces and poor generative performance. 

A major cause of such failings is when the geometries of the prior and target data do not agree. We explore this issue of \textit{prior mismatch} and previous treatments of it in Section \ref{sec:motivation}.
% Part of the reason is the observed lack of an objective measure of generative model quality \cite{theis2015note}.
Given their connection to manifolds, it is natural to draw from classical approaches in manifold learning to improve deep generative models. One of the most principled methods is kernel-based manifold learning
\cite{kernelpca,roweis2000nonlinear,belkin2002laplacian}. This involves embedding data drawn from a manifold $X \subset \mathcal{M}_X$ into a space spanned by the leading eigenfunctions of a kernel on $\mathcal{M}_X$. We focus specifically on \textit{diffusion maps}, where \cite{DM:lafon} show that normalizations of the kernel define a diffusion process that has a uniform stationary distribution over the data manifold. Therefore, drawing from this stationary distribution samples uniformly from the data manifold. This property was used in \cite{sugar} to smoothly interpolate between missing parts of the manifold. However, despite its strong theoretical guarantees, \textit{diffusion maps} are poorly equipped for large scale generative modeling as they do not scale well with dataset size. Moreover, acquiring the inverse mapping from the embedding space --- a crucial component of a generative model --- is traditionally a very expensive procedure \cite{l1,preimage,mika}.

In this paper we address issues in variational inference and manifold learning by combining ideas from both. The theory in manifold learning allows us to recognize and correct \textit{prior mismatch}, whereas variational inference provides a method to construct a generative model, which also offers an efficient approximation to the inverse \textit{diffusion map}.

\textbf{Our contributions:} \textbf{1}) We introduce the locally bi-Lipschitz property, a necessary condition of a homeomorphism, for measuring the stability of a mapping between latent and data distributions. \textbf{2}) 
We introduce variational diffusion autoencoders (VDAEs), a class of variational autoencoders that, instead of directly reconstructing the input, have an encoder-decoder that approximates one discretized time-step of the diffusion process on the data manifold (with respect to a user defined kernel $k$). \textbf{3}) We prove approximation theoretic bounds for deep neural networks learning such diffusion processes, and show that these networks define random walks with certain desirable properties, including well-defined transition and stationary distributions. \textbf{4}) Finally, we demonstrate the utility of the VDAE framework on a set of real and synthetic datasets, and show that they have superior performance and satisfy the locally bi-Lipschitz property.

% \begin{figure}
% \includegraphics{img/diagram.png}

% \caption{Generated points from VDAE at various diffusion times. Top row: Synthetic loop dataset. Bottom row: MNIST handwritten digits dataset.}
% \label{fig:images}
% \end{figure}
\begin{figure}[tb!]
\begin{center}
%\framebox[4.0in]{$\;$}
\includegraphics[width=.9\textwidth]{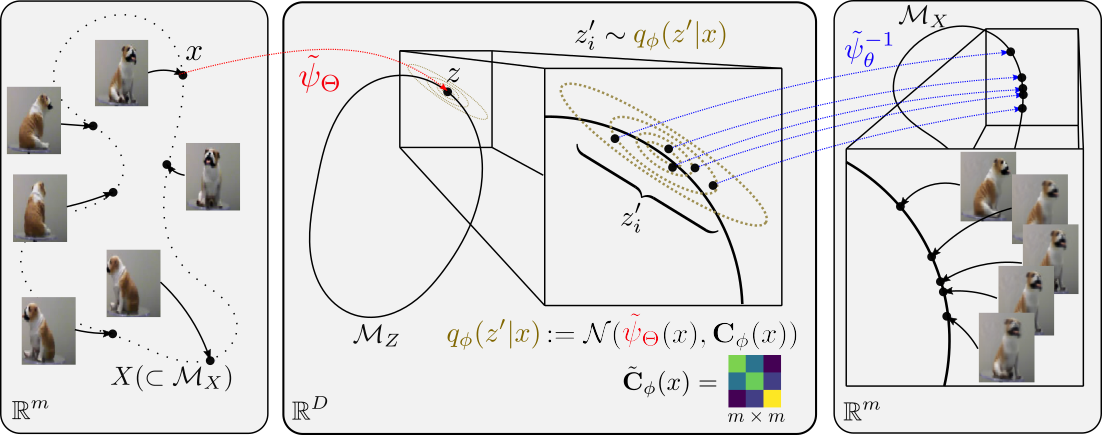}
\end{center}
\caption{A diagram depicting one step of the diffusion process modeled by the variational diffusion autoencoder (VDAE). The \textit{diffusion} and \textit{inverse diffusion maps} $\psi, \psi^{-1}$, as well as the covariance $\mathbf{C}$ of the random walk on $\mathcal{M}_Z$, are all approximated by neural networks. Images on the leftmost panel are actually generated by our method.}
\end{figure}

\section{Background}
\label{sec:background}
% In both cases, it can be shown that a homeomorphism would not exist between the latent prior and the target distribution. [NEED TO REFINE THE TWO FACTORS.] Both produce the same outcome; the manifold cannot be learned, and the resulting model suffers.
% % may either have trouble covering the whole support of the true distribution or generate many points off the support.

% \subsection{Variational Inference}
\textbf{Variational inference} (VI, \cite{vi_1,vi_2}) combines Bayesian statistics and latent variable models to approximate some probability density $p(x)$. VI exploits a latent variable structure in the assumed data generation process, that the observations $x \sim p(x)$ are conditionally distributed given unobserved latent variables $z$. By modeling the conditional distribution, then marginalizing over $z$, as in
\begin{equation}
% $
    \label{eq:marginalization}
    p_\theta(x) = \int_z p_\theta(x|z)p(z)dz,
% $
\end{equation}
we obtain the model evidence, or likelihood that $x$ could have been drawn from $p_\theta(x)$. Maximizing the model evidence (Eq. \ref{eq:marginalization}) leads to an algorithm for finding likely approximations of $p(x)$. The cost of computing this integral scales exponentially with the dimension of $z$ and thus becomes intractable with high latent dimensions. Therefore we replace the model evidence (Eq. \ref{eq:marginalization}) with the evidence lower bound (ELBO):
\begin{align*}
    \log p_\theta(x)
    \geq - D_{KL}(q(z|x)||p(z)) + \mathbb{E}_{z \sim q(z|x)}[\log p_\theta(x|z)]
    , \numberthis \label{eq:elbo}
\end{align*}
where $q(z|x)$ is usually an approximation of $p_\theta(z|x)$. Maximizing the ELBO is sped up by taking stochastic gradients \cite{svi}, and further accelerated by learning a global function approximator $q_\phi $ in an autoencoding structure \cite{kingma2013auto}.

% \subsection{Diffusion Maps} \label{sec:diffusionmaps}
\textbf{Diffusion maps} \cite{DM:lafon} refer to a class of kernel methods that perform non-linear dimensionality reduction on a set of observations $X \subseteq \mathcal{M}_X$, where $\mathcal{M}_X$ is the assumed data manifold equipped with measure $\mu$. Let  $x, y \in X$; given a symmetric and non-negative kernel $k$,
\textit{diffusion maps} involve analyzing the induced random walk on the graph of $X$, where the transition probabilities $P(y|x)$ are captured by the probability kernel
% \begin{equation}
%     p(x,y) = \frac{k(x,y)}{d(x)},
% \end{equation}
$p(x,y) = k(x,y)/d(x),$
where $d(x) = \int_X k(x,y) d\mu(y)$ is the weighted degree of $x$.
The diffusion map itself is defined as
% \begin{equation}
$
    \psi_D(x) := [
\lambda_1 f_1(x),  
\lambda_2 f_2(x), 
..., 
\lambda_D f_D(x)],
$
% \end{equation}
where $\{f_i\}_{1\leq i\leq D}$ and $\{\lambda_i\}_{1 \leq i \leq D}$ are the first $D$ eigenfunctions and eigenvalues of $p$. 
An important construction in \textit{diffusion maps} is the \textit{diffusion distance}:
\begin{equation}
    D(x, y)^2 = \int (p(x, u) - p(y, u))^2 \frac{d\mu(u)}{\pi(u)},
\end{equation}
where $\pi(u) = d(u)/\sum_{z \in X} d(z)$ is the stationary distribution of $u$. Intuitively, $D(x,y)$ measures the difference between the diffusion processes emanating from $x$ and $y$. A key property of $\psi_D$ is that it embeds the data $X \in \mathbb{R}^m$ into the Euclidean space $\mathbb{R}^D$ so that the diffusion distance is approximated by Euclidean distance (up to relative accuracy $\frac{\lambda_D}{\lambda_1}$).
Therefore, the arbitrarily complex random walk induced by $k$ on $\mathcal{M}_X$ becomes an isotropic Gaussian random walk on $\psi(\mathcal{M}_X)$. 

\textbf{SpectralNet} \cite{shaham2018spectralnet} is a neural network approximation of the \textit{diffusion map} $\psi_D$ that enjoys a major computational speedup. 
% Classically, $\psi_D$ could only be computed via the eigendecomposition of $k$, which is an $O(n^3)$ time operation. This restricts the application of \textit{diffusion maps} to small datasets. On larger datasets the Nystr{\"o}m extension \cite{NIPS2000_1866} is often applied, but it requires a careful selection of landmark points. However, \cite{shaham2018spectralnet} proposed an algorithm that approximates the function $\psi_D$ itself with deep neural networks. 
The eigenfunctions $f_1, f_2, \dots, f_D$ that compose $\psi_D$ are learned by optimizing a custom loss function that stochastically maximizes the Rayleigh quotient for each $f_i$ while enforcing the orthogonality of all $f_i \in \{f_n\}_{n=1}^D$ via a custom orthogonalization layer. As a result, the training and computation of $\psi$ is linear in dataset and model size (as opposed to $O(n^3)$). We will use this algorithm to obtain our diffusion embedding prior.

%\textcolor{red}{Summarize Jones Maggioni idea and main theorem - XC will add}

\textbf{Locally bi-Lipschitz coordinates by kernel eigenfunctions.} The construction of local coordinates of Riemannian manifolds $\mathcal{M}_X$ by eigenfunctions of the diffusion kernel is analyzed in \cite{jones2008manifold}.
They establish, for all $x \in \mathcal{M}_X$, the existence of some neighborhood $U(x)$ and $d$ spectral coordinates given $U(x)$ that define a bi-Lipschitz mapping from $U(x)$ to $\mathbb{R}^d$. With a smooth compact Riemannian manifold, we can let $U(x) = B(x, \delta r_\text{in})$, where $\delta$ is some constant and the \textit{inradius} $r_\text{in}$ is the radius of the largest ball around $x$ still contained in $\mathcal{M}_X$. Note that $\delta$ is uniform for all $x$,
whereas the indices of the $d$ spectral coordinates as well as the local bi-Lipschitz constants may depend on $x$ and are order $O(r_\text{in}^{-1})$.
For completeness we give a simplified statement of the \cite{jones2008manifold} result in the Appendix.

Using the compactness of the manifold,
one can always cover the manifold with $m$ many neighborhoods (geodesic balls) on which the bi-Lipschitz property in \cite{jones2008manifold} holds. 
As a result, there are a total of $D$ spectral coordinates, $D \le m d$
(in practice $D$ is much smaller than $md$, since the selected spectral coordinates in the proof of \cite{jones2008manifold} tend to be low-frequency ones, and thus the selection on different neighborhoods tend to overlap),
such that on each of the $m$ neighborhoods,
there exists a subset of $d$ spectral coordinates out of the $D$ ones which are bi-Lipschitz on the neighborhood.  We observe empirically that the bi-Lipschitz constants can be bounded uniformly from below and above (see Section \ref{sec:empirical_bilip}).

\begin{figure}[t]
\begin{center}
\subfigure[$\mathtt{t} = 1$]{
\includegraphics[width=0.23\textwidth,height=.15\textwidth]{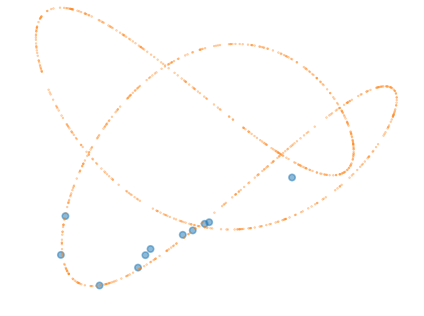}}
\subfigure[$\mathtt{t} = 2$]{
\includegraphics[width=0.23\textwidth,height=.15\textwidth]{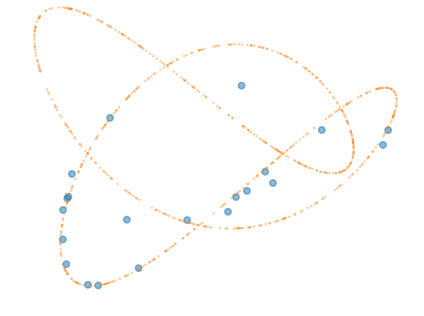}}
\subfigure[$\mathtt{t} = 4$]{
\includegraphics[width=0.23\textwidth,height=.15\textwidth]{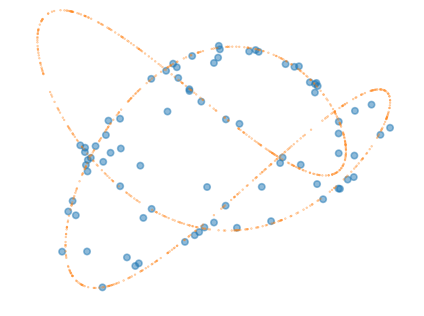}}
\subfigure[$\mathtt{t} = 8$]{
\includegraphics[width=0.23\textwidth,height=.15\textwidth]{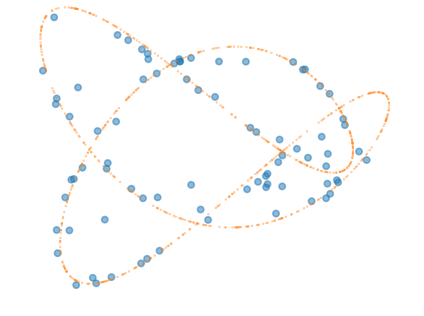}}
\end{center}
\caption{An example of the diffusion random walk simulated by our method on a 3D loop dataset. $\mathtt{t}$ is the number of steps taken in the random walk.}
\label{fig:randomwalk}
\end{figure}

\section{Motivation and related work}
\label{sec:motivation}
In this section we justify the key idea of our method: diagnosing and correcting \textit{prior mismatch}, a failure case of VAE and GAN training when $p(z)$ and $p(x)$ are not topologically isomorphic. Intuitively, we would like the latent distribution to have three nice properties: (1) \textbf{\textit{realizability}}, that every point in the data distribution can be realized as a point in the latent distribution; (2) \textbf{\textit{validity}}, that every point in the latent distribution maps to a unique valid point in the data distribution (even if it is not in the training set); and (3) \textbf{\textit{smoothness}}, that points in the latent distribution vary in the intrinsic coordinate system in some smooth and coherent way.

These properties are precisely those enjoyed by a latent distribution that is homeomorphic to the data distribution. \textit{Validity} implies injectivity, \textit{realizability} implies surjectivity, \textit{smoothness} implies continuity; and a mapping between topological spaces that is injective, surjective, and continuous is a homeomorphism. Therefore, studying algorithms that encourage approximations of homeomorphisms is of fundamental interest.

Though the Gaussian distribution for $p(z)$ is mathematically elegant and computationally expedient, there are many datasets for which it is ill-suited. Spherical distributions are known to be superior for modeling directional data \cite{fisher1993statistical,mardia2014statistics}, which can be found in fields as diverse as bioinformatics \cite{bioinformatics}, geology \cite{geology}, materials science \cite{materialscience}, natural image processing \cite{writtendigit}, and many preprocessed datasets\footnote{Any dataset where the data points have been normalized to be unit length becomes a subset of a hypersphere.}.  For data supported on more complex manifolds, the literature is sparse, even though it is well-known that data often lie on such manifolds \cite{manifoldhypothesis1,manifoldhypothesis2}. In general, any manifold-supported distribution that is not globally homeomorphic to Euclidean space will not satisfy conditions (1-3) above.

Previous research on alleviating \textit{prior mismatch} exists in various forms, and has focused on increasing the family of tractable latent distributions for generative models. \cite{sphere1,sphere2} consider VAEs with the von-Mises Fisher distribution, a geometrically hyperspherical prior, and \cite{vampprior} consider mixtures of priors. \cite{dvae} propose a method that, like our method, also samples from the prior via a diffusion process over a manifold. However, their method requires very explicit knowledge of the manifold (including its projection map, scalar curvature constant, and volume), and give up an exact estimation of the KL divergence. \cite{razavi2019preventing} avoids mode collapse by lower bounding the KL-divergence term away from zero to avoid overfitting. Similarly, \cite{mao2019mode} focuses on avoiding mode collapse by using class-conditional generative models, however it requires label supervision and does not provide any guarantees that the latent space generated is homeomorphic to the data space. Finally, \cite{bvae} propose the re-scaling of various terms in the ELBO to augment the latent space --- often to surprisingly great effect on latent feature discovery --- but are restricted to the case where the latent features are independent.

While these methods expand the repertoire of feasible priors, they all require explicit user knowledge of the data topology. On the other hand, our method allows the user to be agnostic to this choice of topology; they only need to specify an affinity kernel $k$ for local pairwise similarities. We achieve this by employing ideas from both \textit{diffusion maps} and variational inference, resulting in a fully data-driven approach to latent distribution selection in deep generative models.

%\section{Measure}
%Here we propose the \textit{locally bi-Lipschitz measure}.

%Demonstrate the bi-Lipschitz constants empirically on data, vs normal VAE- To do this, take nearest neighbors of a point in manifold distance, and compute the upper and lower lipschitz constants

\section{Method}
\label{sec:method}
In this section we propose the variational diffusion autoencoder (VDAE), a class of generative models built from ideas in \textit{variational inference} and \textit{diffusion maps}. Given the data manifold $\mathcal{M}_X$, observations $X \subset \mathcal{M}_X$, and a kernel $k$, VDAEs model the geometry of $X$ by approximating a random walk over the latent diffusion manifold $\mathcal{M}_Z := \psi(\mathcal{M}_X)$. The model is trained by maximizing the \textit{local evidence}: the evidence (i.e. log-likelihood) of each point given its random walk neighborhood. Points are generated from the trained model by sampling from $\pi$, the stationary distribution of the resulting random walk.

Starting from some point $x \in X$, we can think of one step of the walk as the composition of three functions: \textbf{1}) the approximate diffusion map $\widetilde{\psi}_\omega: \mathcal{M}_X \rightarrow \mathcal{M}_Z$ parameterized by $\omega$, \textbf{2}) the stochastic function that samples from the diffusion random walk $z' \sim q_\phi(z'|x) = \mathcal{N}(\widetilde{\psi}_\omega(x), \widetilde{\mathbf{C}}_\phi(x))$ on $\mathcal{M}_Z$, and \textbf{3}) the approximate inverse diffusion map $\widetilde{\psi}^{-1}_\theta: \mathcal{M}_Z \rightarrow \mathcal{M}_X$ that generates $x' \sim p(x'|z') = \mathcal{N}(\widetilde{\psi}^{-1}_\theta(z'), cI)$ where $c$ is a fixed, user-defined hyperparameter usually set to 1.

Note that Euclidean distances in $\mathcal{M}_Z$ approximate single-step random walk distances on $\mathcal{M}_X$ due to properties of the diffusion map embedding (see \hyperref[sec:background]{Section 2} and \cite{DM:lafon}). These properties are inherited by our method via the SpectralNet algorithm, since $\widetilde{\psi}_\omega|_{\mathcal{M}_X}:\mathcal{M}_X \rightarrow \mathcal{M}_Z$ satisfies the \textit{locally bi-Lipschitz property}. This bi-Lipschitz property also reduces the need for regularization, and leads to guarantees of the ability of the VDAE to avoid posterior and mode collapse (see Section \ref{theory}).

In short, to model a diffusion random walk over $\mathcal{M}_Z$, we must learn the functions $\widetilde{\psi}_\omega, \widetilde{\psi}^{-1}_\theta$, and $ \widetilde{\mathbf{C}}_\phi$ that approximate the diffusion map, the inverse diffusion map, and the covariance of the random walk on $\mathcal{M}_Z$, at all points $z \in \mathcal{M}_Z$. SpectralNet gives us $\widetilde{\psi}_\omega$. To learn $\widetilde{\psi}^{-1}_\theta$ and $\widetilde{\mathbf{C}}_\phi$, we use variational inference.

\subsection{The lower bound}
Formally, let us define $U_x := B_d(x, \delta) \cap \mathcal{M}_X$, where $B_d(x, \delta)$ is the $\delta$-ball around x with respect to $d(\cdot, \cdot)$, the diffusion distance on $\mathcal{M}_Z$. For each $x \in X$ we define the \textit{local evidence} of $x$ as
\begin{equation}
    \mathbb{E}_{x' \sim p(x'|x)|_{U_x}} \log p_\theta(x'|x),
\end{equation}
where $p(x'|x)|_{U_x}$ restricts $p(x'|x)$ to $U_x$. This gives the \textit{local evidence lower bound}
\begin{align*}
    &\log p_\theta(x'|x) 
    \geq
    \underbrace{- D_{KL}(q_\phi(z'|x) || p_\theta(z'|x))}_\text{divergence from true diffusion probabilities} + \underbrace{\mathbb{E}_{z' \sim q_\phi(z'|x)}\log p_\theta(x'|z')}_\text{neighborhood reconstruction error}, \label{eq:manifold_elbo} \numberthis
\end{align*}
which produces the empirical loss function $\Tilde{\mathcal{L}}_{\text{VDAE}} = -D_{KL}(q_\phi(z'|x)||p_\theta(z'|x)) + \log p_\theta(x'|z_i')$,
where $z_i' = g_{\phi, \Theta}(x, \epsilon_i)$, $\epsilon_i \sim \mathcal{N}(0, I)$. The function $g_{\phi,\Theta}$ is deterministic and differentiable, depending on $\widetilde{\psi}_\omega$ and $\widetilde{\mathbf{C}}_\phi$, that generates $q_{\phi}$ by the reparameterization trick\footnote{Though $q$ depends on $\phi$ and $\omega$, we will use $q_\phi := q_{\phi,\omega}$ to be consistent with existing VAE notation and to indicate that $\omega$ is not learned by variational inference.}.
% $x'$ is drawn from the random walk neighborhood of $x$, i.e., $x' \sim p_\theta(x'|x) \approx \widetilde{\psi}^{-1}_\theta(q_\phi(z'|x))]$ which we approximate empirically by
% \begin{equation}
%     x' \approx \argmin_{y \in A} |\widetilde{\psi}_\omega(y) - z'|_d^2 \;, \; \; \; \; z' \sim q_\phi(z'|x), \label{eq:empirical_neighbor}
% \end{equation}
% where $A \subseteq X$ is the training batch.

% \begin{algorithm}
% \label{alg:training}
% \caption{VDAE training}
% \begin{algorithmic}%[1]
% \STATE $\omega, \phi, \theta \gets \text{Initialize parameters}$
% \STATE Obtain parameters $\omega$ for the approximate diffusion map $\widetilde{\psi}_\omega$ via SpectralNet \cite{shaham2018spectralnet}
% \WHILE{not converged}
% \STATE $A \gets$ Random batch from $X$
% \FOR{$x \in A$}
% \STATE $z \gets p_\phi(z'|\widetilde{\psi}_\omega(x))$ \COMMENT{Random walk step}
% \STATE $x' \gets \argmin_{y \in A \setminus \{x\}} |\widetilde{\psi}_\omega(y) - z'|_d^2$ \COMMENT{Find batch neighbors}
% \STATE $g \gets g + \frac{1}{|A|}\nabla_{\phi, \theta} \log p_\theta(x'|x)$ \COMMENT{Compute Eq. \eqref{eq:empirical_manifold_elbo}}
% \ENDFOR
% \STATE Update $\phi, \theta$ using $g$
% \ENDWHILE
% \end{algorithmic}
% \end{algorithm}

\subsection{The sampling procedure}
\label{sec:sampling}
% Recall that the diffusion process has a stationary distribution that is uniform on the manifold. By modeling the diffusion process in such a manner, we have access to not only $p(x'|x)$, the random walk on $\mathcal{M}_X$, but also $p(z'|z)$, the random walk on $\mathcal{M}_Z$. Therefore, given any point $z \in \mathcal{M}_Z$ we are able to sample \textit{uniformly} from the latent manifold by drawing repeatedly from $p(z'|z)$. This gives 

% Note that, by modeling the diffusion process, 
% allowing the prior to be the uniform $X$ comes at a cost: sampling from an arbitrary manifold prior is no longer a single-step procedure as it is with simpler priors. But we claim that the cost is not too high, since we have access to $\psi$. Recall that, by the diffusion map framework, $k$ defined on $\mathcal{M}_Z$ generates a random walk on $\mathcal{M}_X$.

Composing $q_\phi(z'|x)$($\approx p_\theta(z'|x)$) with $p_\theta(x'|z')$ gives us an approximation of $p_\theta(x'|x)$. Then the simple, parallelizable, and fast random walk based sampling procedure naturally arises: initialize with an arbitrary point on the manifold $x_0 \in \mathcal{M}_X$ (e.g. from the dataset $X$), pick suitably large $N$, and for $n=1,\dots, N$ draw $x_n \sim p(x|x_{n-1})$.  See Section \ref{randomwalk} for examples of points drawn from this procedure.

\subsection{A practical implementation}
We now introduce a practical implementation, considering the case where $\widetilde{\psi}_\omega(x)$, $q_\phi(z'|x)$ and $p_\theta(x'|z')$ are neural network functions.

The \textbf{neighborhood reconstruction error} $\mathbb{E}_{z' \sim q_\phi(z'|x)}\log p_\theta(x'|z')$ should be differentiated from the \textit{self} reconstruction error in VAEs, i.e. reconstructing $x'$ vs $x$. Since $q_\phi(z'|x)$ models the neighborhood of $\widetilde{\psi}_\omega(x)$, we may sample $q_\phi$ to obtain $z'$ (the neighbor of $x$ in the latent space). Assuming $\psi^{-1}$ exists, we have $x' \sim p_\theta(x'|x) (\approx \widetilde{\psi}^{-1}_\theta(q_\phi(z'|x)))$. To make this practical, we can approximate $x'$ by finding the closest data point to $x'$ in random walk distance (due to the aforementioned advantages of the latent space). %This gives \eqref{eq:empirical_neighbor}.
In other words, we approximate empirically by
\begin{equation}
    x' \approx \argmin_{y \in A} |\widetilde{\psi}_\omega(y) - z'|_d^2 \;, \; \; \; \; z' \sim q_\phi(z'|x), \label{eq:empirical_neighbor}
\end{equation}
where $A \subseteq X$ is the training batch.

On the other hand, the \textbf{divergence of random walk distributions},\\ $- D_{KL}(q_\phi(z'|x) || p_\theta(z'|x))$, can be modeled simply as the divergence of two Gaussian kernels defined on $\mathcal{M}_Z$. Though $p_\theta(z'|x)$ is intractable, the diffusion map $\psi$ gives us the diffusion embedding $Z$, which is an approximation of the true distribution of $p_\theta(z'|x)$ in a neighborhood around $z = \psi(x)$. We estimate the first and second moments of this distribution in $\mathbb{R}^D$ by computing the local Mahalanobis distance of points in the neighborhood. Then, by minimizing the KL divergence between $q_\phi(z'|x)$ and the one implied by this Mahalanobis distance, we obtain the loss:
\begin{align*}
    & - D_{KL}(q_\phi(z'|x) || p_\theta(z'|x))
    = -\log \frac{|\alpha \Sigma_*|}{|\widetilde{\mathbf{C}}_\phi|} + d - tr\{(\alpha \Sigma_*)^{-1}\widetilde{\mathbf{C}}_\phi\}, \numberthis{} \label{eq:practical_loss}
\end{align*}
where $\widetilde{\mathbf{C}}_\phi(x)$ is a neural network function, $\Sigma_*(x) = \text{Cov}(B_{d}(\psi(x), \delta) \cap Z)$ is the covariance of the points in a neighborhood of $z = \psi(x) \in Z$, and $\alpha$ is a scaling parameter controlling the random walk step size. Note that the covariance $\widetilde{\mathbf{C}}_\phi(x)$ does not have to be diagonal, and in fact is most likely not. 
% Empirically, we observe that these neighborhoods need not be very small to obtain good approximations (see \ref{fig:alpha}).
Combining Eqs. \ref{eq:empirical_neighbor} and \ref{eq:practical_loss} we obtain \hyperref[alg:training]{Algorithm 1}.

Since we use neural networks to approximate the random walk induced by the composition of $q_\phi(z'|x)$ and $p_\theta(x'|z')$, the generation procedure is highly parallelizable. This leads naturally to a sampling procedure for this random walk (\hyperref[alg:sampling]{Algorithm 2}). We observe that the random walk enjoys rapid mixing properties --- it only takes several iterations of the random walk to sample from all of $\mathcal{M}_Z$ \footnote{For all experiments in Section \ref{experiments}, the number of steps required to draw from $\pi$ is less than 10.}.

Finally, we describe a practical method for computing the local bi-Lipschitz property. (In Section \ref{sec:empirical_bilip} we then perform comparisons with this method.) Let $Z$ and $X$ be the latent and generated data distributions of our model $f$ (i.e. $f : Z \rightarrow X$). We define, for each $z \in Z$ and $k \in \mathbb{N}$, the function $\mathtt{bilip}_k(z)$:
\begin{align*}
    \mathtt{bilip}_k(z) = \min \{K:
    \frac{1}{K} \leq \frac{d_x(f(z), f(z'))}{d_Z(z, z')} \leq K\},
\end{align*}
for all $z' \in U_{z,k} \cap Z$, where $d_X$ and $d_Z$ are metrics on $X$ and $Z$, and $U_{z,k}$ is the $k$-nearest neighborhood of $z$. Intuitively, increasing values of $K$ characterize an increasing tendency to \textit{stretch} or \textit{compress} regions of space. By analyzing statistics of the local bi-Lipschitz measure at all points in a latent space $Z$, we gain insight into how well-behaved a mapping $f$ is.

\begin{algorithm*}
\begin{algorithmic}%[1]
%\subcaptionbox{first}
\STATE $\omega, \phi, \theta \gets \text{Initialize parameters}$
\STATE Obtain parameters $\omega$ for the approximate diffusion map $\widetilde{\psi}_\omega$ via SpectralNet \cite{shaham2018spectralnet}
\WHILE{not converged}
\STATE $A \gets$ Random batch from $X$
\FOR{$x \in A$}
\STATE $z \gets p_\phi(z'|\widetilde{\psi}_\omega(x))$ \COMMENT{Random walk step}
\STATE $x' \gets \argmin_{y \in A \setminus \{x\}} |\widetilde{\psi}_\omega(y) - z'|_d^2$ \COMMENT{Find batch neighbors}
\STATE $g \gets g + \frac{1}{|A|}\nabla_{\phi, \theta} \log p_\theta(x'|x)$ \COMMENT{Compute Eq. \eqref{eq:manifold_elbo}}
\ENDFOR
\STATE Update $\phi, \theta$ using $g$
\ENDWHILE
\end{algorithmic}
\caption{VDAE training}
\label{alg:training}
\end{algorithm*}

\subsection{Comparison to variational inference (VI)}
Traditional VI involves maximizing the joint log-evidence of each data point $x_i$ in a given dataset via the ELBO (see \ref{sec:background}). Our method differs in both the training and evaluation steps.

In training, our setup is the same as above, except our likelihood is a conditional likelihood $p(x'|x)|_{U_x}$, where $x'$ is in the diffusion neighborhood of $x$. Thus we maximize the local log-evidence of each data point $\mathbb{E}_{x' \sim p(x'|x_i)} \log p_\theta(x'|x_i)$,
which can be lower bounded by Eq. \eqref{eq:manifold_elbo}. Thus our prior is $p(z'|x)$ and our posterior is $p(z'|x',x) = p(x',z'|x)/p(x'|x)$, and we train an approximate posterior $q_\phi(z'|x)$ and a recognition model $p_\theta(x'|z')$.

In evaluation, we draw from the stationary distribution $p(z')$ of the diffusion random walk on the latent manifold $\mathcal{M}_z = \psi(\mathcal{M}_x)$. We then leverage the latent variable structure of our model to draw a sample $x = p_\theta(x'|z')p(z')$, where $p_\theta(x'|x_i)$ is the recognition model.

\section{Theory} \label{theory}

In this section, we show that the desired diffusion and inverse diffusion maps $\psi: \mathcal{M}_X \rightarrow \mathcal{M}_Z$ and $\psi^{-1}:\mathcal{M}_Z \rightarrow \mathcal{M}_X$ can be approximated by neural networks, where the network complexity is bounded by quantities related to the intrinsic geometry of the manifold.

\begin{algorithm*}[tb!]

%\label{alg:sampling}
%\begin{multicols}{2}

\begin{algorithmic}
%\title{a}
\caption{VDAE sampling}
 \STATE $X_0 \gets $ Initialize with points $X_0 \subset X$; $t \gets 0$
 %\STATE 
 \WHILE{$p(X_0) \not \approx \pi$}
 \FOR{$x_t \in X$}
 \STATE $z_{t+1} \sim p_\phi(z'|\widetilde{\psi}_\omega(x_t))$ \COMMENT{Random walk step}
 \STATE $x_{t+1} \sim p_\theta(x|z_{t+1})$ 
 \COMMENT{Map back to input space}
 \ENDFOR
%  \State $X_{t+1} \gets X_{\text{subsample}} \subset X_{t+1}'$
 \STATE $t \gets t + 1$
\ENDWHILE
\end{algorithmic}
\label{alg:sampling}

\end{algorithm*}
The capacity of the encoder $\widetilde{\psi}$ has already been considered in \cite{shaham2018provable} and \cite{mishne2017diffusion}. Thus we focus on the capacity of the decoder $\widetilde{\psi}^{-1}$.
The following theorem is proved in Appendix \ref{proofs}, based on the result in \cite{jones2008manifold}.
\begin{thm}\label{thm:main}
Let $\mathcal{M}_X \subset \mathbb{R}^m$ be a smooth $d$-dimensional manifold, $\psi(\mathcal{M}_X) \subset \mathbb{R}^{D}$ be the diffusion map for $D\ge d$ large enough to have a subset of coordinates that are locally bi-Lipschitz.  Let $\mathbf{X} = \begin{bmatrix} X_1, & ..., & X_m
\end{bmatrix}$ be the set of all $m$ extrinsic coordinates of the manifold.  Then there exists a sparsely-connected ReLU network $f_N$, with $4DC_{\mathcal{M}_X}$ nodes in the first layer, $8dmN$ nodes in the second layer, and $2mN$ nodes in the third layer, and $m$ nodes in the output layer, such that
    \begin{equation}
        \left\|\mathbf{X}(\psi(x)) - {f_N}(\psi(x)) \right\|_{L^2(\psi(\mathcal{M}_X))} \leq \sqrt{m}C_\psi / \sqrt{N},
    \end{equation}
    where the norm is interpreted as
    $\|F\|^2_{L^2(\psi(\mathcal{M}))} := \int \|F(\psi(x))\|_2^2 d\psi(x)$. Here $C_\psi$ depends on how sparsely $X(\psi(x)) \big|_{U_i}$ can be represented in terms of the ReLU wavelet frame on each neighborhood $U_i$, and $C_{\mathcal{M}_X}$ on the curvature and dimension of the manifold $\mathcal{M}_X$.
\end{thm}

Thm \ref{thm:main} guarantees the existence and size of a decoder network for learning a manifold. Together with the main theorem in \cite{shaham2018provable}, we obtain guarantees for both the encoder and decoder on manifold-valued data. The proof is built on two properties of ReLU neural networks: 1) their ability to split curved domains into small, almost Euclidean patches, 2) their ability to build differences of bump functions on each patch, which allows one to borrow approximation results from the theory of wavelets on spaces of homogeneous type.  The proof also crucially uses the bi-Lipschitz property of the diffusion embedding \cite{jones2008manifold}. The key insight of Thm \ref{thm:main} is that, because of the bi-Lipschitz property, the coordinates of the manifold in the ambient space $\mathbb{R}^m$ can be thought of as functions of the diffusion coordinates.  We show that because each coordinate function $X_i$ is Lipschitz, the ReLU wavelet coefficients of $X_i$ are necessarily $\ell^1$.  This allows us to use the existing guarantees of \cite{shaham2018provable} to complete the desired bound.

We also discuss the connections between the distribution at each point in diffusion map space, $q_\phi(z|x)$, and the result of this distribution after being decoded through the decoder network $f_N(z)$ for $z\sim q_\phi(z|X)$.  Similar to  \cite{singer2008non}, we characterize the covariance matrix $Cov(f_N(z)) := \mathbb{E}_{z\in q_\phi(z|x)}[f_N(z) f_N(z)^T]$.
The following theorem is proved in Appendix \ref{proofs}.

\begin{figure}[tb!]
    \centering
    \subfigure{
     \includegraphics[width=2.5cm]{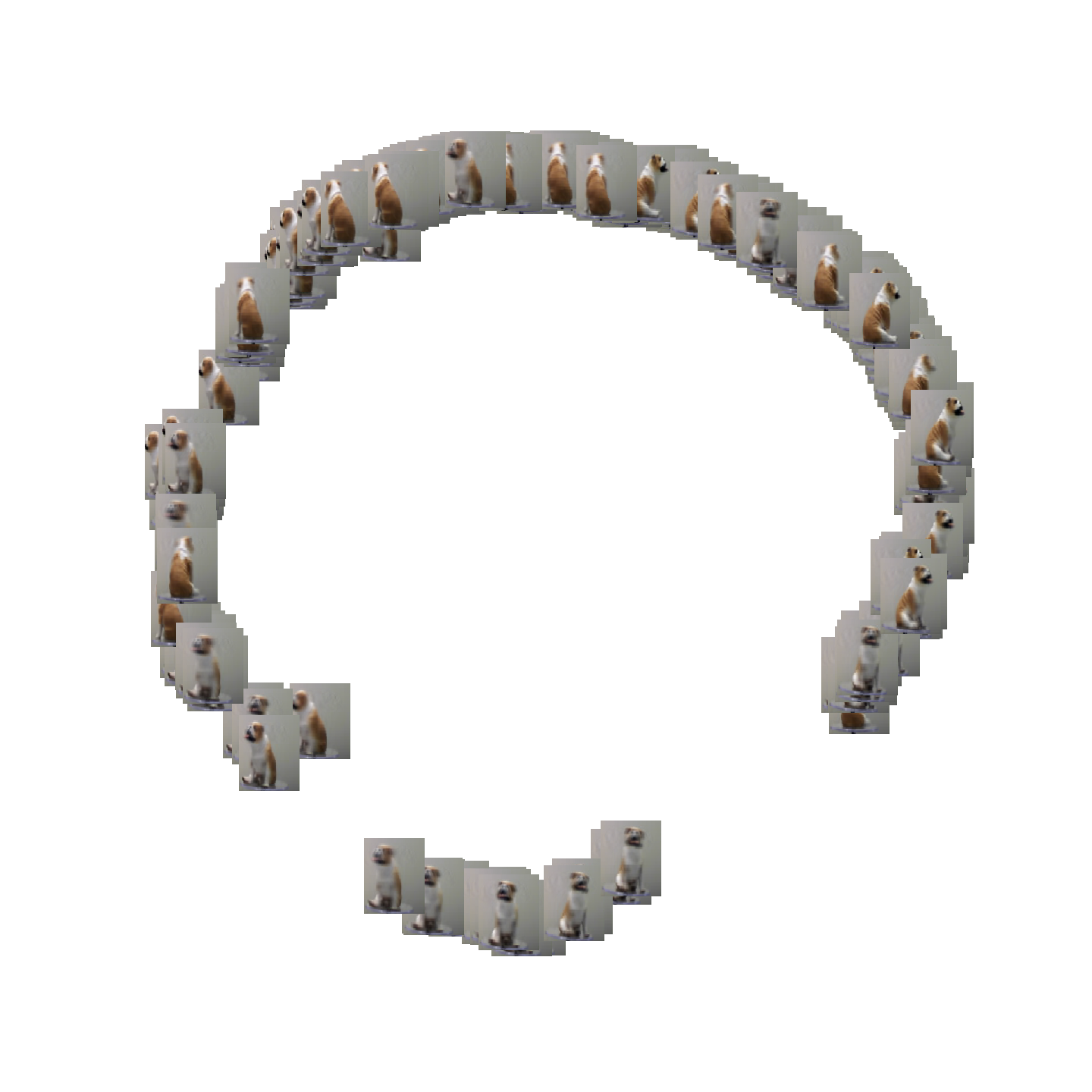}
     }
     \subfigure{
     \includegraphics[width=2.5cm]{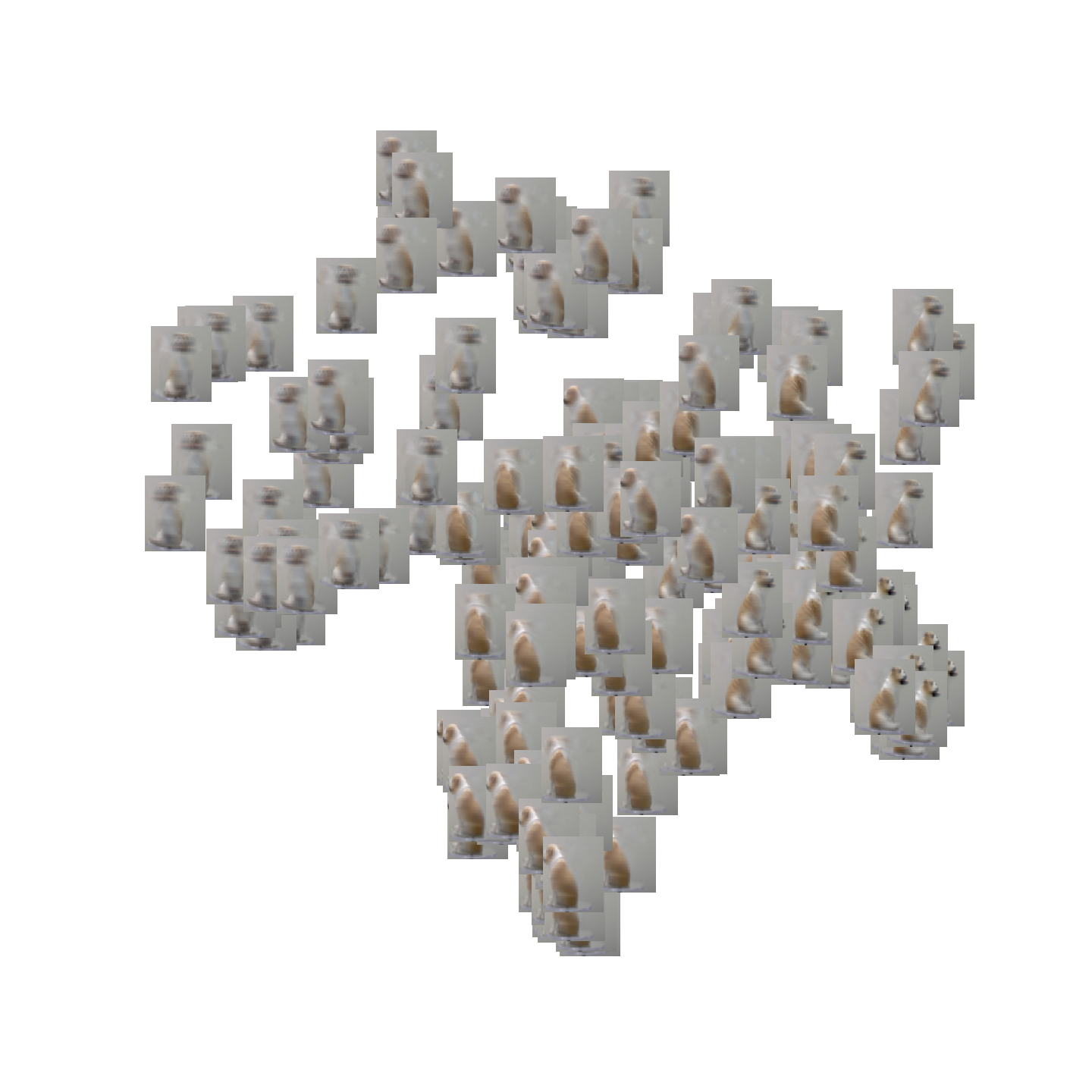}
     }
     \subfigure{
     \includegraphics[width=2.5cm]{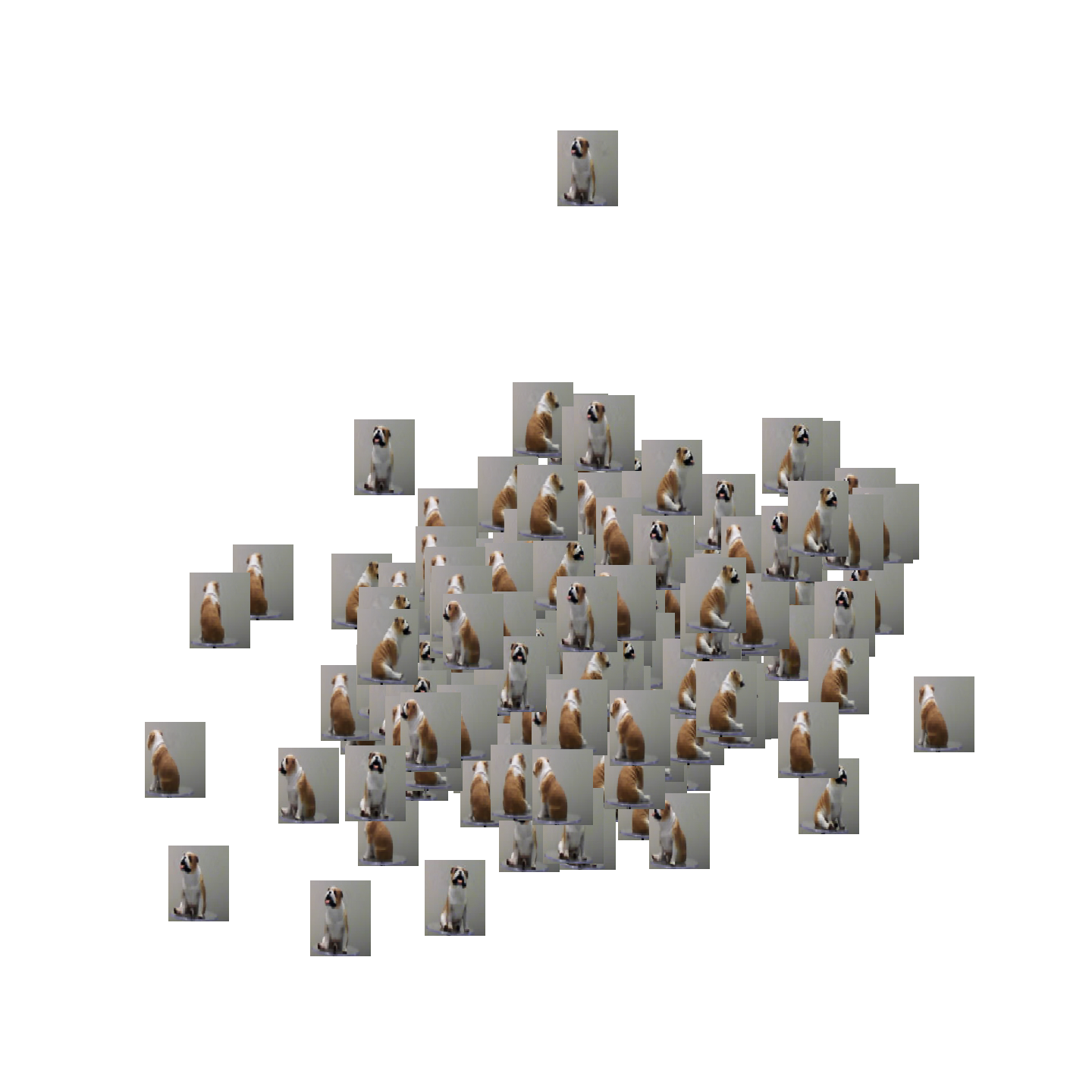}
     }
     \subfigure{
     \includegraphics[width=2.5cm]{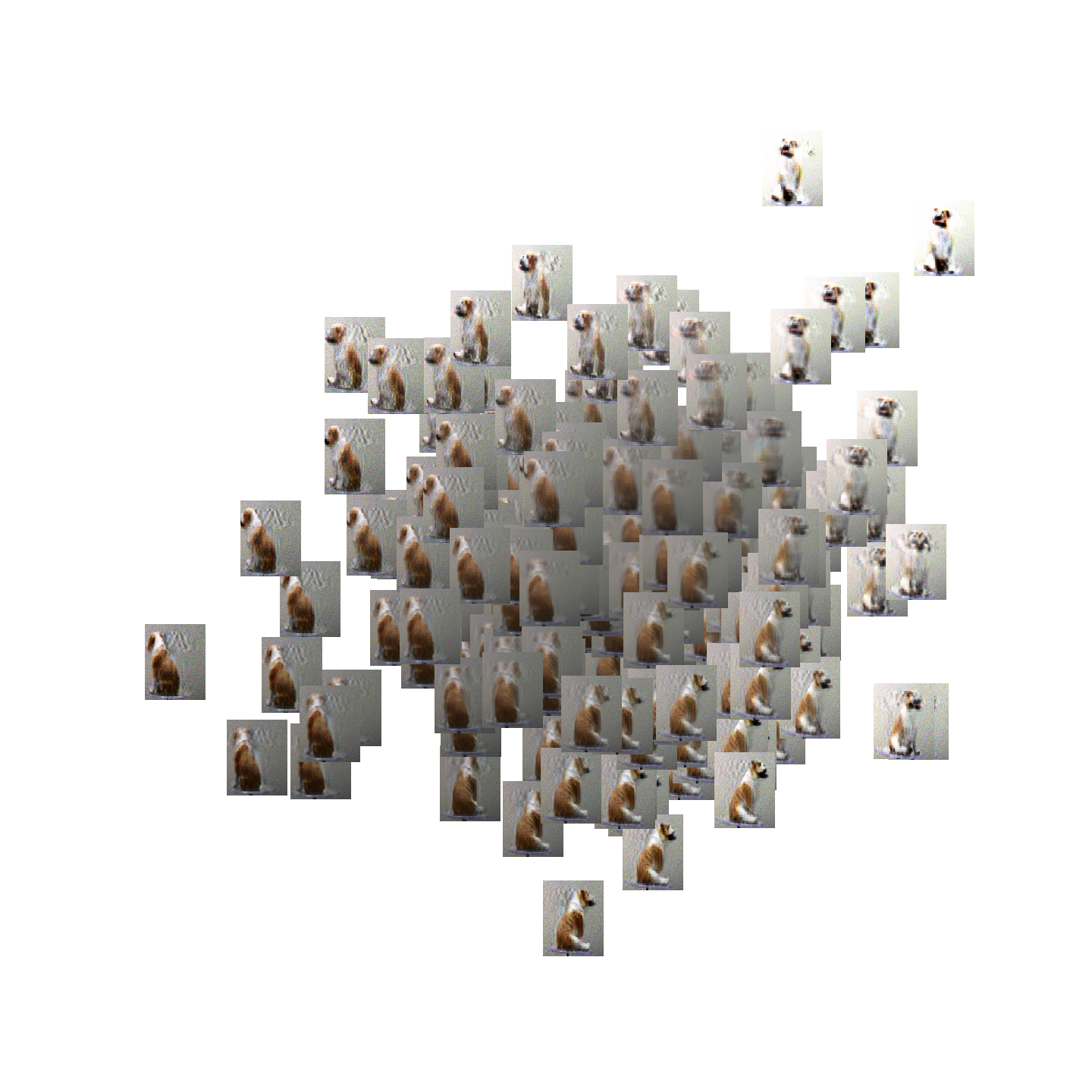}
     }
\caption{We consider the rotating bulldog example. Images are drawn from the latent distribution and plotted in terms of the 2D latent space of each model. From left to right: VDAE, SVAE, $\beta$-VAE, WGAN.}
\label{fig:rotatingbulldog}
\end{figure}

\begin{thm}
Let $f_N$ be a neural network approximation to $\mathbf{X}$ as in Thm \ref{thm:main}, such that it approximates the extrinsic manifold coordinates.
Let $C\in \mathbb{R}^{m\times m}$ be the covariance matrix $C=\mathbb{E}_{z\in q_\phi(z|x)}[f_N(z) f_N(z)^T]$.  Let $q_\phi(z|x) \sim N(\psi(x),\Sigma)$ with small enough $\Sigma$ that there exists a patch $U_{z_0}\subset \mathcal{M}$ around $z_0$ satisfying the bi-Lipschitz property of \cite{jones2008manifold}, and such that
$Pr(z\sim q_\phi(z|x) \not\in \psi(U_{z_0}))<\epsilon$.  Then the number of eigenvalues of $C$ greater than $\epsilon$ is at most $d$, and $C = J_{z_0}\Sigma J_{z_0}^T + O(\epsilon)$ where $J_{z_0}$ is the $m\times D$ Jacobian matrix at $z_0$.
\label{thm:burst}
\end{thm}

Thm \ref{thm:burst} establishes the relationship between the covariance matrices used in the sampling procedure and their image under the decoder $f_N$ to approximate $\psi^{-1}$.  Similar to \cite{singer2008non}, we are able to sample according to a multivariate normal distribution in the latent space.  Thus, the resulting cloud in the data space is distorted (to first order) by the local Jacobian of the map $f_N$.  The key insight of Thm \ref{thm:burst} is from combining this idea with the observation of \cite{jones2008manifold}: that $\psi^{-1}$ depends locally only on $d$ of the coordinates in the $D$ dimensional latent space.

\section{Experimental results} \label{experiments}
In this section we explore various properties of the VDAE and compare it against several deep generative methods on a selection of real and synthetic datasets. Unless otherwise noted, all comparisons are against the Wasserstein GAN (WGAN), $\beta$-VAE, and hyperspherical VAE (SVAE). Each model is trained with the same architecture across all experiments (see Section \ref{sec:architectures}).
\subsection{Video generation with rigid-body motion}

We first consider the task of generating new frames from videos of rigid-body motion, and examine the latent spaces of videos with known topological structure to demonstrate the homeomorphic properties of the VDAE. We consider two examples, the rotating bulldog example \cite{roy2} and the COIL-20 dataset. \cite{nene1996columbia}. 

The rotating bulldog example consists of $200$ frames of a color video (each frame is $100 \times 80 \times 3$) of a spinning figurine. The rotation of the bulldog and the fixed background create a data manifold that is topologically circular, corresponding to the single degree of variation (the rotation angle parameter) in the dataset. For all methods we consider a 2 dimensional latent space. In Fig. \ref{fig:rotatingbulldog} we present 300 generated samples by displaying them on a scatter plot with coordinates corresponding to their latent dimensions $z_1$ and $z_2$. In the Appendix table \ref{tab:fid}, we evaluate the quality of the generated images using the Frechet inception distance (FID).

% \begin{figure}[t]
%     \centering
%     \subfigure{
%      \includegraphics[width=2cm]{img/coil20/vdae.png}
%     }
%     \subfigure{
%     \includegraphics[width=2cm]{img/coil20/svae.png}
%     }
%     \subfigure{
%     \includegraphics[width=2cm]{img/coil20/tsnecoilvae-new.png}
%     }
%     \subfigure{
%     \includegraphics[width=2cm]{img/coil20/tsnecoilgan-new2.png}
%     }
% \caption{A tSNE embedding of $360$ generated images from COIL-20 data set. Pictured from left-right: VDAE, SVAE, $\beta$-VAE, WGAN. A larger figure is included in Appendix \ref{fig:coil20_full}.} \label{fig:coil20}
% \end{figure}

The COIL-20 data set consists of 360 images of five different rotating objects displayed against on a black background (each frame is $448\times 416\times 1$). This yields several low dimensional manifolds, one for each object, and results in a difficult data set for traditional generative models given its small size and the complex geometric structure. For all comparisons, we use $10$ dimensional latent space. The resulting images are embedded with tSNE and plotted in Fig. \ref{fig:coil20}. Note that, while other methods generate images that topologically mimic the fixed latent distribution of the model (e.g. $\mathcal{N}(0,I_d)$, $\text{Uniform}(0,1)^{d}$), our method generates images that remain true to the actual topological structure of the dataset.

\begin{figure}[tb!]
\begin{center}
\subfigure{
\includegraphics[width=0.15\textwidth]{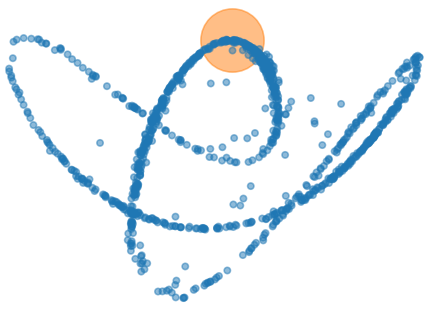}}
\subfigure{
\includegraphics[width=0.15\textwidth]{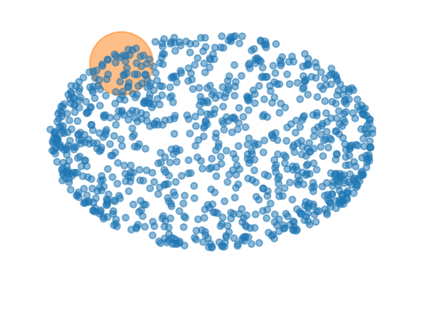}}
\subfigure{
\includegraphics[width=0.15\textwidth]{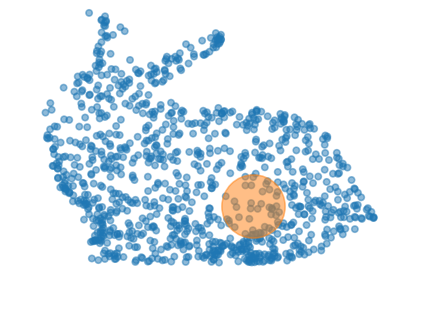}}
\subfigure{
\includegraphics[width=0.15\textwidth]{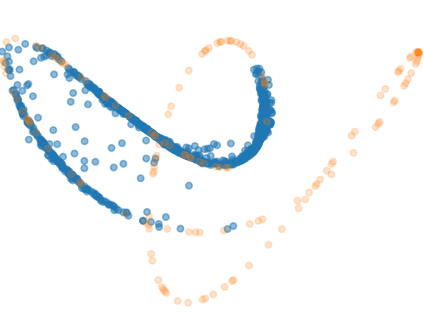}}
\subfigure{
\includegraphics[width=0.15\textwidth]{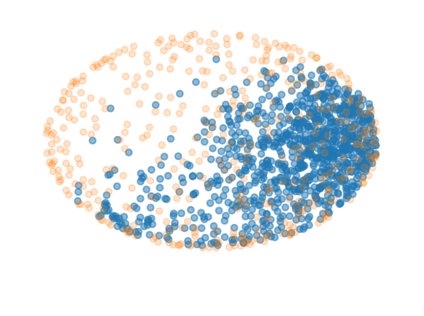}}
\subfigure{
\includegraphics[width=0.15\textwidth]{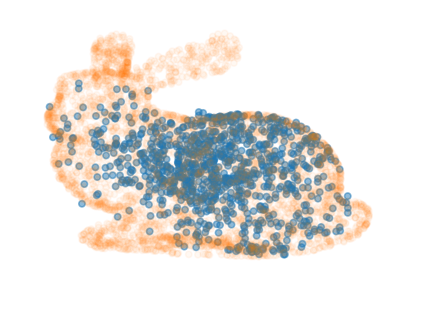}}
\end{center}
\caption{From left to right, the first three scatterplots show examples of distributions reconstructed from a random walk on $\mathcal{M}_Z$ (via \hyperref[alg:sampling]{Algorithm 2}) given a single seed point drawn from $X$. The next three are examples of a single burst drawn from $p_\theta(x|z)$.
The distributions are a loop (a, d), sphere (b, e), and the Stanford bunny (c, f).}
\label{fig:diffusion_burst}
\end{figure}

\subsection{Data generation from uniformly sampled manifolds}\label{randomwalk}
In the next experiment, we visualize the results of the sampling procedure in \hyperref[alg:training]{Algorithm 2} on three synthetic manifolds. As discussed in \ref{sec:sampling}, we randomly select an initial seed point, then recursively sample from $p_\theta(x'|x)$ to simulate a random walk on the manifold.

In fig. \ref{fig:diffusion_burst} (a-c) for three different manifolds, the location of the initial seed point is highlighted, then 20 steps of the random walk are taken, and the resulting generated points are displayed.  The generated points remain on the manifold even after this large number of resampling iterations, and the distribution of sampled points converges to a uniform stationary distribution on the manifold. Moreover, we observe that this stationary distribution is reached quickly, within 5-10 iterations. In (d-f) of the same Fig. \ref{fig:diffusion_burst}, we show $p_\theta(x'|x)$ by drawing a large number of points from a single-step random walk starting from the same seed point. As can be seen, a single step of $p_\theta(x'|x)$ covers a large part of the latent space.

\subsection{Cluster conditional data generation}
In this section, we deal with the problem of generating samples from data with multiple clusters in an unsupervised fashion (i.e. no a priori knowledge of the cluster structure).  Clustered data creates a problem for many generative models, as the topology of the latent space (i.e. normal distribution) differs from the topology of the data space with multiple clusters.  

First we show that our method is capable of generating new points from a particular cluster given an input point from that cluster.  This generation is done in an unsupervised fashion, which is a different setting from the approach of conditional VAEs \cite{sohn2015learning} that require training labels.  We demonstrate this property on MNIST \cite{mnist} in Figure \ref{fig:clusterMNIST}, and show that the newly generated points after a short diffusion time remain in the same class as the seeded image.

\begin{figure}[tb!]
\centering
\subfigure{
    \includegraphics[width=1.8cm, height=1.8cm]{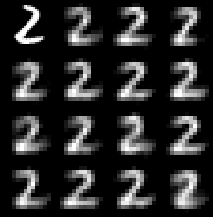}}
\subfigure{
    \includegraphics[width=1.8cm, height=1.8cm]{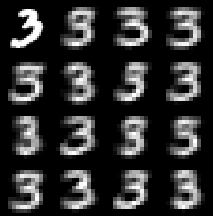}}
\subfigure{
    \includegraphics[width=1.8cm, height=1.8cm]{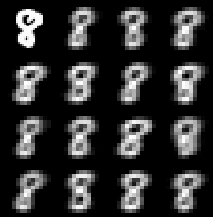}}
\subfigure{
    \includegraphics[width=1.8cm, height=1.8cm]{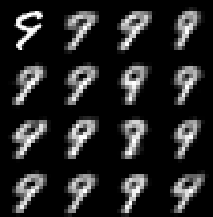}}
\caption{An example of cluster conditional sampling with our method, given a seed point (top left of each image grid). The VDAE is able to produce examples via the random walk that stay approximately within the cluster of the seed point, without any supervised knowledge of the cluster.}
\label{fig:clusterMNIST}
\end{figure}

The problem of addressing differing topologies between the data and the latent space of a generative model has been acknowledged in recent works on rejection sampling \cite{azadi2018discriminator,turner2018metropolis}.  Rejection sampling of neural networks consists of generating a large collection of samples using a standard GAN, and then designing a probabilistic algorithm to decide in a {\it post-hoc} fashion whether the points were truly in the support of the data distribution $p(x)$.

In the following experiment, we compare to a standard example in the literature for rejection sampling in generative models (see \cite{azadi2018discriminator}). The data consists of nine bounded spherical densities with significant minimal separation, lying on a $5 \times 5$ grid. A GAN struggles to avoid generating points in the gaps between these densities, and thus requires the post-sampling rejection analysis described in \cite{azadi2018discriminator}. Conversely, our model creates a latent space that separates each of these clusters into their own coordinates and generates only points that in the neighborhood of the support of $p(x)$. Figure \ref{fig:rejectionsampling} shows that this results in significantly fewer points generated in the gaps between clusters. Our VDAE architecture is described in \ref{sec:architectures}, GAN and DRS-GAN architectures are as described in \cite{azadi2018discriminator}.

\subsection{Quantitative comparisons of generative models}
\label{sec:empirical_bilip}

For this comparison, we consider seven datasets: three synthetic (circle, torus, Stanford bunny \cite{bunny}) four involving natural images (rotating bulldog, Frey faces, MNIST, COIL-20). The $\beta$ parameter in the $\beta$-VAE is optimized via a cross validation procedure. see Appendix for a complete description of the datasets. We report the mean and standard deviation of the Gromov-Wasserstein distance \cite{memoli2011gromov} and median bi-Lipschitz over 5 runs in Table \ref{tab:gwAndBilipTable}. We further evaluate the results using kernel Maximum Mean Discrepancy \cite{gretton2012kernel}, see Table \ref{tab:mmdtable} in the Appendix. 

By constraining our latent space to be the diffusion embedding of the data, our method finds a mapping that automatically enjoys the homeomorphic properties of an ideal mapping, and this is reflected in the low values of the local bi-Lipschitz constant. Conversely, other methods do not consider the topology of the data in the prior distribution. This is especially apparent in the $\beta$-VAE and SVAE, which must generate from the entirety of the input distribution $X$ because they minimize a reconstruction loss. Interestingly, the mode collapse tendency of GANs alleviate the pathology of the bi-Lipschitz constant by allowing the GAN to focus on a subset of the distribution --- but this comes at the cost of collapse to a few modes of the dataset. Our method is able to reconstruct the entirety of $X$ while simultaneously maintaining a low local bi-Lipschitz constant.
\begin{figure}[t!]
\centering
\subfigure{
    \includegraphics[width=1.8cm, height=1.8cm]{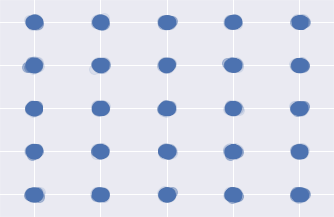}}
\subfigure{
    \includegraphics[width=1.8cm, height=1.8cm]{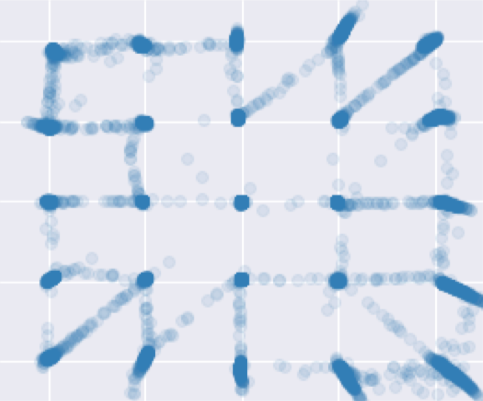}}
\subfigure{
    \includegraphics[width=1.8cm, height=1.8cm]{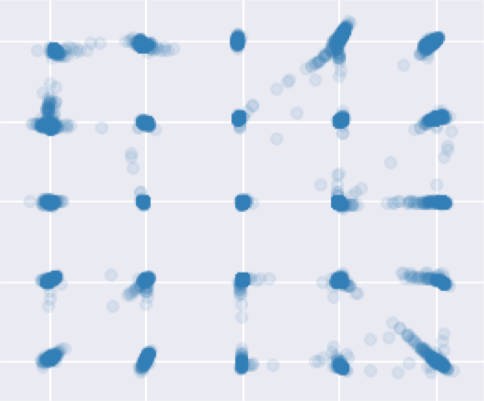}}
\subfigure{
    \includegraphics[width=1.8cm, height=1.8cm]{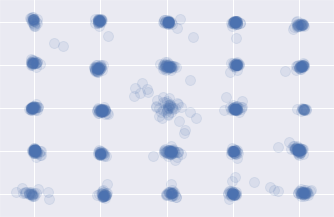}}
\caption{Comparison of samples from our method against several others on a $5 \times 5$ Gaussian grid. Left-right are original data, GAN, DRS-GAN, and VDAE (our method). GAN and DRS-GAN samples taken from \cite{azadi2018discriminator}.}
\label{fig:rejectionsampling}
\end{figure}

\begin{table}[b!]
\tiny
 \centering
    \begin{tabular}{|l|l|l|l|l|}
    \hline
    \textbf{G-W} & WGAN   & $\beta$-VAE   &
    SVAE & VDAE \\
    \hline
    Circle & 14.9(6.8) & 46.1(9.7) &  7.9(2.2)  & \textbf{2.6(1.3)} \\
    Torus & 6.4 (1.9) & 11.7(1.6) & 23.4(2.8)  & \textbf{4.9 (0.5)} \\
    Bunny & 11.4(3.9) & 32.8(5.9) & 14.3(5.5) & \textbf{2.9(1.1)}  \\
    Bulldog & 117.3(8.4) & 61.3(9.7) &  53.9(7.6) & \textbf{15.3(1.7)} \\
    Frey  & 18.1(2.9) & 19.8(4.6) & 13.4(3.6) & \textbf{9.7(3.3)} \\
    MNIST & \textbf{3.6(0.9}) & 10.2(3.3) & 15.2(4.9) & 14.4(3.5) \\
    COIL-20 & 16.5(2.4) & 23.8(5.9) &  32.1(4.9)  & \textbf{11.8(2.1)} \\
    \hline
    \end{tabular}%
    \quad
    \begin{tabular}{|l|l|l|l|l|}
    \hline
    \textbf{biLip} & {WGAN} & {$\beta$-VAE} & {SVAE} & {VDAE} \\
    \hline
    Circle & 4.6 & 3.7  &  3.6  & \textbf{3.1} \\
    Torus & \textbf{3.3}  & 7.9  & 9.5 & 4.8 \\
    Bunny & 5.6   & 34.4 &  35.6  & \textbf{5.5} \\
    Bulldog & 17.4  & 7.6   &  12.9   & \textbf{6.8} \\
    Frey  & 37    & 33.3  & 39.4 & \textbf{29.7} \\
    MNIST & 1.9   & \textbf{1.6}   &  6.7  & 8.4 \\
    COIL-20 & 4.7   & 3.8   & 8.4 & \textbf{3.1} \\
    \hline
    \end{tabular}%
    \caption{Left: means and standard deviations of the Gromov-Wasserstein (G-W) distance between original and generated samples. Right: medians of the bi-Lipschitz measure.}
  \label{tab:gwAndBilipTable}%
\end{table}%

\section{Discussion}

In this work, we have shown that VDAEs provide an intuitive, effective, and mathematically rigorous solution to \textit{prior mismatch}, which is a common cause for posterior collapse in latent variable models. Unlike prior works, we do not require user specification of the prior --- our method infers the prior geometry directly from the data, and we observe that it achieves state-of-the-art results on several real and synthetic datasets. Finally, our work points to several directions for future research: (1) can we leverage recent architectural advances to VAEs to further improve VDAE performance, and (2) can we leverage manifold learning techniques to improve latent representations in other methods?

\section*{Acknowledgements}
HL, AC, and XC are supported by NSF (DMS-1819222 and -1818945).
XC and OL are also partly supported by NIH (RGM131642A); AC by NSF (DMS-2012266) and the Russell Sage Foundation (2196); XC by NSF (DMS-1820827), NIH (R01GM131642), and the Sloan Foundation; and OL by NIH (R01HG008383).

\nocite{sajjadi2018assessing}

%\begin{table}[t]
 %   \centering
  %  \begin{tabular}{|c|c|c|c|c|}
   % \hline
    %  {\bf Validation Measure}   & {\bf WGAN} & {\bf VAE} & {\bf SVAE} & {\bf VDAE}  \\
    %  \hline
     %    &  & & & \\
    %\hline
    %\end{tabular}
    %\caption{\textcolor{red}{Caption}}
    %\label{tab:bilip}
%\end{table}

%\bibliographystyle{splncs04}
%\bibliographystyle{natbib}
% \bibliography{example_paper}
% \bibliographystyle{icml2020}
\bibliographystyle{splncs04}
\bibliography{example_paper}

\newpage
\appendix
\section{Appendix}
\setcounter{figure}{0} \renewcommand{\thefigure}{A.\arabic{figure}}
\setcounter{table}{0} \renewcommand{\thetable}{A.\arabic{table}}
\setcounter{equation}{0} \renewcommand{\theequation}{A.\arabic{equation}}
\setcounter{page}{1} \renewcommand{\thepage}{A.\arabic{page}}

\subsection{Derivation of Local Evidence Lower Bound (Eq. \ref{eq:manifold_elbo})}
\label{eq:derive-evidence-lower-bound}
We begin with taking the log of the random walk transition likelihood,
\begin{align}
    \log p_\theta(x'|x) &= \log \int_z' p_\theta(x',z'|x) dz' \\
    &= \log \int_z p_\theta(x'|z',x) p(z'|x) \frac{q(z')}{q(z')} dz' \\
    &= \log \mathbb{E}_{z' \sim q(z')} \left[ p_\theta(x'|z',x) \frac{p(z'|x)}{q(z')} \right] \\
    &\geq \mathbb{E}_{z' \sim q(z')} \left[\log p_\theta(x'|z',x) \right] + \mathbb{E}_{z' \sim q(z')} \left[\log \frac{p(z'|x)}{q(z')} \right] \\
    &\geq \mathbb{E}_{z' \sim q(z')} \left[\log p_\theta(x'|z',x) \right] + D_{KL} [q(z') || p(z'|x)]
\end{align}

where $q(z')$ is an arbitrary distribution. We let $q(z')$ to be the conditional distribution $q(z'|x)$. Furthermore, if we make the simplifying assumption that $p_\theta(x'|z',z) = p_\theta(x'|z')$, then we obtain Eq. \ref{eq:manifold_elbo}
\begin{equation}
    \log p_\theta(x'|x) 
    \geq - D_{KL}(q_\phi(z'|x) || p_\theta(z'|x)) + \mathbb{E}_{z' \sim q_\phi(z'|x)}\log p_\theta(x'|z').
\end{equation}

% \section{Results in \cite{jones2008manifold}}
\subsection[Results in]{Results in \cite{jones2008manifold}}
To state the result in \cite{jones2008manifold}, we need the following set-up:

(C1) ${\cal M}$ is a $d$-dimensional smooth compact manifold, possibly having boundary, equipped with a smooth (at least $C^2$) Riemannian metric $g$;

We denote the geodesic distance by $d_{\cal M}$, and the geodesic ball centering at $x$ with radius $r$ by $B_{\cal M}(x,r)$.
Under (C1), 
for each point $x \in {\cal M}$, 
there exists $r_{\cal M}(x)$ which is the inradius, 
that is, $r$ is the largest number s.t. $B_{\cal M}(x,r)$ is contained ${\cal M}$. 
%Let the $U$ be the geodesic ball with radius $r_{\cal M}(z)$ centering at $x$, the neighborhood $U$ can be charted by the exponential map at $x$. 

Let $\triangle_{\cal M}$ be the Laplacian-Beltrami operator on ${\cal M}$ with Neumann boundary condition, which is self-adjoint on $L^2(M, \mu)$, $\mu$ being the Riemannian volume given by $g$.
Suppose that ${\cal M}$ is re-scaled to have volume 1.
The next condition we need concerns the spectrum of the manifold Laplacian

(C2) $\triangle_{\cal M}$ has discrete spectrum, 
and the eigenvalues $\lambda_0 \le \lambda_1 \le \cdots$ satisfy the Weyl's estimate, i.e. exists constant $C$ which only depends on ${\cal M}$ s.t.
\[
|\{ j: ~ \lambda_j \le T \}| \le C T^{d/2}.
\]

Let $\psi_j$ be the eigenfunction  associated with $\lambda_j$,
$\{ \psi_j \}_j$ form an orthonormal bases of $L^2(M, \mu)$.
The last condition is

(C3) The heat kernel 
(defined by the heat equation on ${\cal M}$) 
has the spectral representation as 
\[
K_t(x,y) = \sum_{j=0}^\infty e^{-t\lambda_j} \psi_j(x)\psi_j(y).
\]

\begin{thm}[Thm 2 \cite{jones2008manifold}, simplified version]
Under the above setting and assume (C1)-(C2), 
then there are positive constants $c_1, c_2, c_3$ 
which only depend on ${\cal M}$ and $g$, 
s.t. 
for any $x \in {\cal M}$, $r_{\cal M}(x)$ being the inradius,
there are $d$ eigenfunctions of $\triangle_{\cal M}$,
$\psi_{j_1}, \cdots, \psi_{j_d}$,
which collectively give a mapping $\Psi: {\cal M} \to \mathbb{R}^d$ by
\[
\Psi_x (x) = (\psi_{j_1}(x), \cdots,  \psi_{j_d}(x))
\]
satisfying that $\forall y,y' \in B(x, c_1 r_{\cal M}(x))$,
\[
c_2 r_{\cal M}(z)^{-1} d_{\cal M}(y,y')
\le
\|\Psi_x (y) - \Psi_x (y') \|
\le 
c_3 r_{\cal M}(z)^{-1-d/2} d_{\cal M}(y,y').
\]
That is, $\Psi$ is bi-Lipschitz on the neighborhood $B(x, c_1 r_{\cal M}(x)) $
with the Lipschitz constants indicated as above.
The subscript $x$ in $\Psi_x$
emphasizes that the indices $j_1, \cdots, j_d$
may depend on $x$.
\end{thm}

\subsection{Proofs}\label{proofs}
\begin{proof}[of Thm \ref{thm:main}]
The proof of Thm \ref{thm:main} is actually a simple extension of the following Thm, Thm \ref{thm:main_appendix}, which needs to be proved for each individual extrinsic coordinate $X_k$, hence the additional factor of $m$ coming from the $L2$ norm of $m$ functions.
\end{proof}

\begin{thm}\label{thm:main_appendix}
    Let $\mathcal{M} \subset \mathbb{R}^m$ be a smooth $d$-dimensional manifold, $\psi(\mathcal{M}) \subset \mathbb{R}^{D}$ be the diffusion map for $D\ge d$ large enough to have a subset of coordinates that are locally bi-Lipschitz. Let one of the $m$ extrinsic coordinates of the manifold be denoted $X(\psi(x))$ for $x\in \mathcal{M}$. Then there exists a sparsely-connected ReLU network $f_N$, with $4DC_{\mathcal{M}}$ nodes in the first layer, $8dN$ nodes in the second layer, and $2N$ nodes in the third layer, such that
    \begin{equation}
        \|X - {f_N}\|_{L^2(\psi(\mathcal{M}))} \leq \frac{C_\psi}{\sqrt{N}}
    \end{equation}
    where $C_\psi$ depends on how sparsely $X(\psi(x)) \big|_{U_i}$ can be represented in terms of the ReLU wavelet frame on each neighborhood $U_i$, and $C_{\mathcal{M}}$ on the curvature and dimension of the manifold $\mathcal{M}$.
\end{thm}
\begin{proof}[of Thm \ref{thm:main_appendix}]

The proof borrows from the main theorem of \cite{shaham2018provable}.  We adopt this notation and summarize the changes in the proof here.  For a full description of the theory and guarantees for neural networks on manifolds, see \cite{shaham2018provable}. 
  Let  $C_\mathcal{M}$ be the number of neighborhoods $U_i = B(x_i,\delta)\cap \mathcal{M}$ needed to cover $\mathcal{M}$ such that $\forall x,y\in U_i$, $(1-\epsilon)\|x - y\| \le d_\mathcal{M}(x,y) \le (1+\epsilon)\|x-y\|$.    Here, we choose $\delta = \min(\delta_\mathcal{M},\kappa^{-1}\rho)$ where $\delta_\mathcal{M}$ is the largest $\delta$ that preserves locally Euclidean neighborhoods and $\kappa^{-1}\rho$ is the smallest value from \cite{jones2008manifold} such that every neighborhood $U_i$ has a bi-Lipschitz set of diffusion coordinates.  

Because of the locally bi-Lipschitz guarantee from \cite{jones2008manifold}, we know for each $U_i$ there exists an equivalent neighborhood $\widetilde \psi(U_i)$ in the diffusion map space, where $\widetilde \psi(x) = \begin{bmatrix}\psi_{i_1}(x), & ..., & \psi_{i_d}(x)\end{bmatrix}$.  Note that the choice of these $d$ coordinates depends on the neighborhood $U_i$.  Moreover, we know the Euclidean distance on $\psi(U_i)$ is locally bi-Lipschitz w.r.t. $d_\mathcal{M}(\cdot,\cdot)$ on $U_i$.  

First, we note that as in \cite{shaham2018provable}, the first layer of a neural network is capable of using $4D$ units to select the subset of $d$ coordinates $\widetilde \psi(x)$ from $\psi(x)$ for $x\in U_i$ and zeroing out the other $D-d$ coordinates with ReLU bump functions.  Then we can define $X(\widetilde \psi(x)) = X(\psi(x))$ on $x\in U_i$.

Now to apply the theorem from \cite{shaham2018provable}, we must establish that $X\big|_{U_i}: \widetilde\psi(U_i)\rightarrow \mathbb{R}$ can be written efficiently in terms of ReLU functions.    Because of the manifold and diffusion metrics being bi-Lipschitz, we know at a minimum that $\widetilde\psi$ is invertible on $\widetilde\psi(U_i)$.  Because of this invertibility, we will slightly abuse notation and refer to $X(\psi(x)) = X(x)$, where this is understood to be the extrinsic coordinate of the manifold at the point $x$ that cooresponds to $\psi(x)$.
we also know that $\forall x,y\in U_i$, 
\begin{align*}
|X(\widetilde\psi(x)) - X(\widetilde\psi(y))| &= |X(x) - X(y)| \\
&\leq \max_{z\in U_i} \|\nabla X(z)\| d(x,y) \\
&\leq \frac{\max_{z\in U_i} \|\nabla X(z)\|}{1-\epsilon} \|\widetilde\psi(x) - \widetilde\psi(y)\|,
\end{align*}
where $\nabla X(z)$ is understood to be the gradient of $X(z)$ at the point $z\in \mathcal{M}$.
This means $X(\widetilde\psi(x))$ is a Lipschitz function w.r.t. $\widetilde\psi(x)$.
Because $X(\widetilde\psi(x))$ Lipschitz continuous, it can be approximated by step functions on a ball of radius $2^{-\ell}$ to an error that is at most $\frac{\max_{z\in U_i} \|\nabla X(z)\|}{1-\epsilon}  2^{-\ell}$.  This means the maximum ReLU wavelet coefficient is less than $\frac{\max_{z\in U_i} \|\nabla X(z)\|}{1-\epsilon}  (2^{-\ell} + 2^{-\ell+1})$.  This fact, along with the fact that $\widetilde\psi(U_i)$ is compact, gives the fact that on $\widetilde\psi(U_i)$, set of ReLU wavelet coefficients is in $\ell^1$.  And from \cite{shaham2018provable}, if on a local patch the function is expressible in terms of ReLU wavelet coefficients in $\ell^1$, then there is an approximation rate of $\frac{1}{\sqrt{N}}$ for $N$ ReLU wavelet terms. 
\end{proof}

\begin{proof}[of Thm \ref{thm:burst}]
We borrow from \cite{singer2008non} to prove the following result.  Given that the bulk of the distribution $q$ lies inside $\psi(U_{z_0})$, we can consider only the action of $f_N$ on $\psi(U_{z_0})$ rather than on the whole space.  Because the geodesic on $U$ is bi-Lipschitz w.r.t. the Euclidean distance on the diffusion coordinates (the metric on the input space), we can use the results from \cite{singer2008non} and say that on $\psi(U_{z_0})$ the output covariance matrix is characterized by the Jacobian of the function $f_N$ mapping from Euclidean space (on the diffusion coordinates) to the output space, at the point $z_0$. So the covariance of the data lying insize $\psi(U_{z_0})$ is $J_{z_0} \Sigma J^T_{z_0}$, with an $O(\epsilon)$ perturbation for the fact that $\epsilon$ fraction of the data lies outside $\psi(U_{z_0})$.

The effective rank of $C$ being at most $d$ comes from the locally bi-Lipschitz property.  We know $X(\psi(x))$ only depends on the $d$ coordinates $\widetilde \psi(x)$ as in the proof of Thm \ref{thm:main}, which implies $f_N(\psi(x))$ satisfies a similarly property if $f_N$ fully learned $X(\psi(x))$.  Thus, while $J\in \mathbb{R}^{m\times D}$, it is at most rank $d$, which means $J\Sigma J^T$ is at most rank $d$ as well.  

\end{proof}
\subsection{Spectral Net}

\subsection{Additional Experimental Result}
To evaluate the quality of the generated images in the Bulldog dataset, we use the Frechet inception distance (FID). We train the different generative models $5$ times and compute the FID between source and generated images. In table \ref{tab:fid} we present the mean and standard deviations of the FID.

\begin{table}[htb]
  \centering
   \begin{tabular}{|l|l|l|l|l|}
       \hline
    \textbf{FID} & GAN   & VAE   & {SVAE} & {VDAE} \\
        \hline
    Bulldog & 264.4(18.4) & 245.7(14.7) & 400.6 (6.2)  & 144.3(12.6) \\

        \hline
    \end{tabular}%
    \caption{Frechet inception distance (FID) on the Bulldog dataset, mean and standard deviation.}
  \label{tab:fid}%
\end{table}%

\begin{table}[htb]
  \centering
   \begin{tabular}{|l|l|l|l|l|}
       \hline
    \textbf{MMD} & GAN   & VAE   & {SVAE} & {VDAE} \\
        \hline
    Circle & 9.3(11.1) & 8.3(4.4) & 8.1 (4.2)  & 7.3(4.3) \\
    Torus & 12.3 (4.7) & 63.3 (12.9) & 84.5(11.7) & 41.9 (4.1) \\
    Bunny & 175.6(68.6) & 725.8(3.8) & 601.7(41.1) &3.6(0.3) \\
    Bulldog & 741.8(88) & 167.3(16.4) & 213.7(13.1) & 9.68(3.44) \\
    Frey  & 34.9(5.1) & 39.3(6.1) & 29.4 & 47.0 \\
    MNIST & 3.5(0.6) & 27.9(1) & 20.6(1.2)  & 5.79(0.3)  \\
    COIL-20 & 3.3(0.9) & 39.2(9.6) &  55.7(4.7)  & 7.4(1.07) \\
        \hline
    \end{tabular}%
    \caption{Measures of similarity between training data and generated data using Maximum Mean Discrepancy. Comparisons are across a variety of synthetic and real data sets}
  \label{tab:mmdtable}%
\end{table}%

\begin{figure}[t]
    \centering
    %\textcolor{red}{Add figures here}
%    \subfigure[VDAE]{
%     \includegraphics[width=1.8cm]{img/bulldog/buldog_gen3.png}
%     }
%     \subfigure[SVAE]{
%     \includegraphics[width=1.8cm]{img/bulldog/buldog_spherical_vae_latent.png}
%     }
    \subfigure[VAE]{
    \includegraphics[width=0.45 \textwidth]{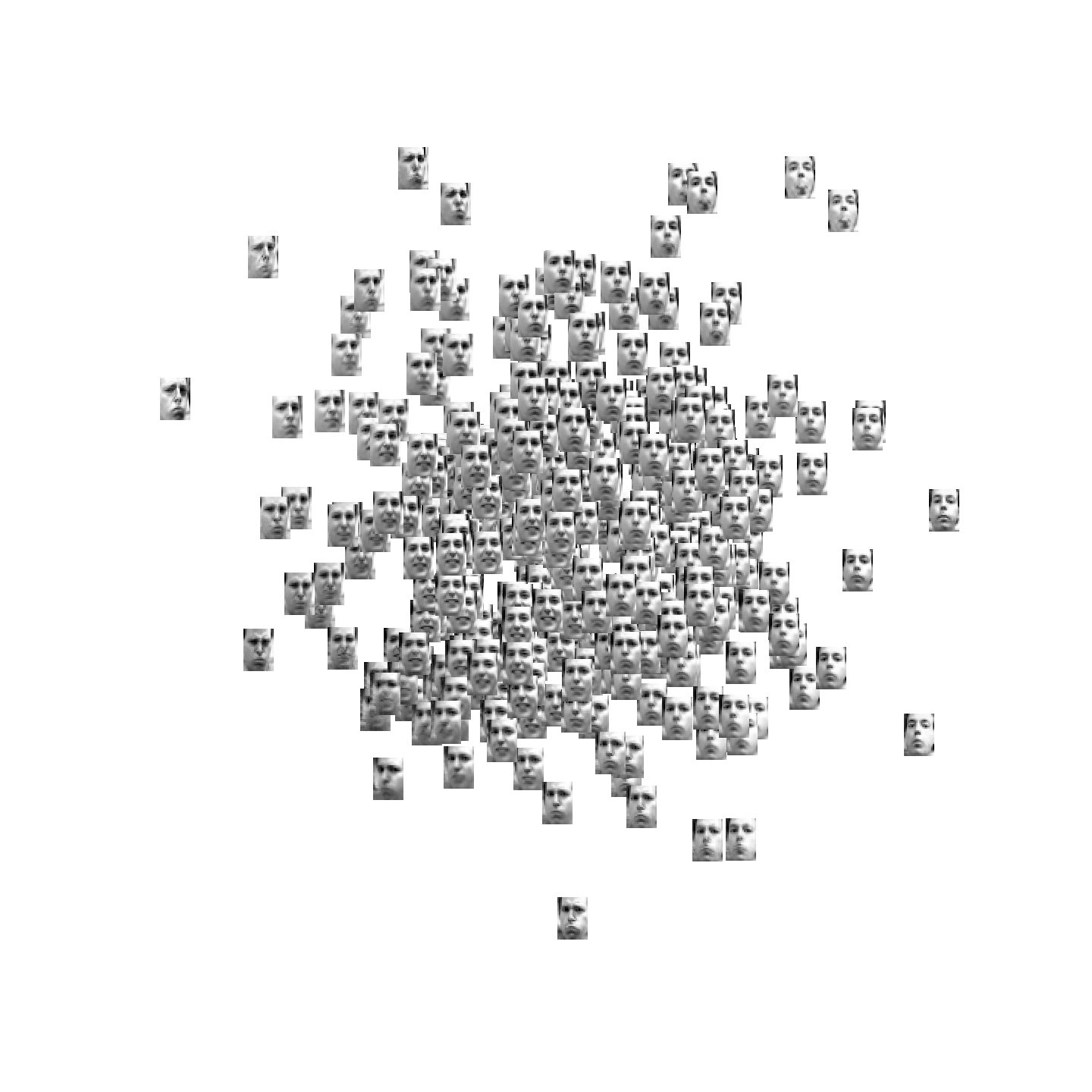}
    }
    \subfigure[GAN]{
    \includegraphics[width=0.45 \textwidth]{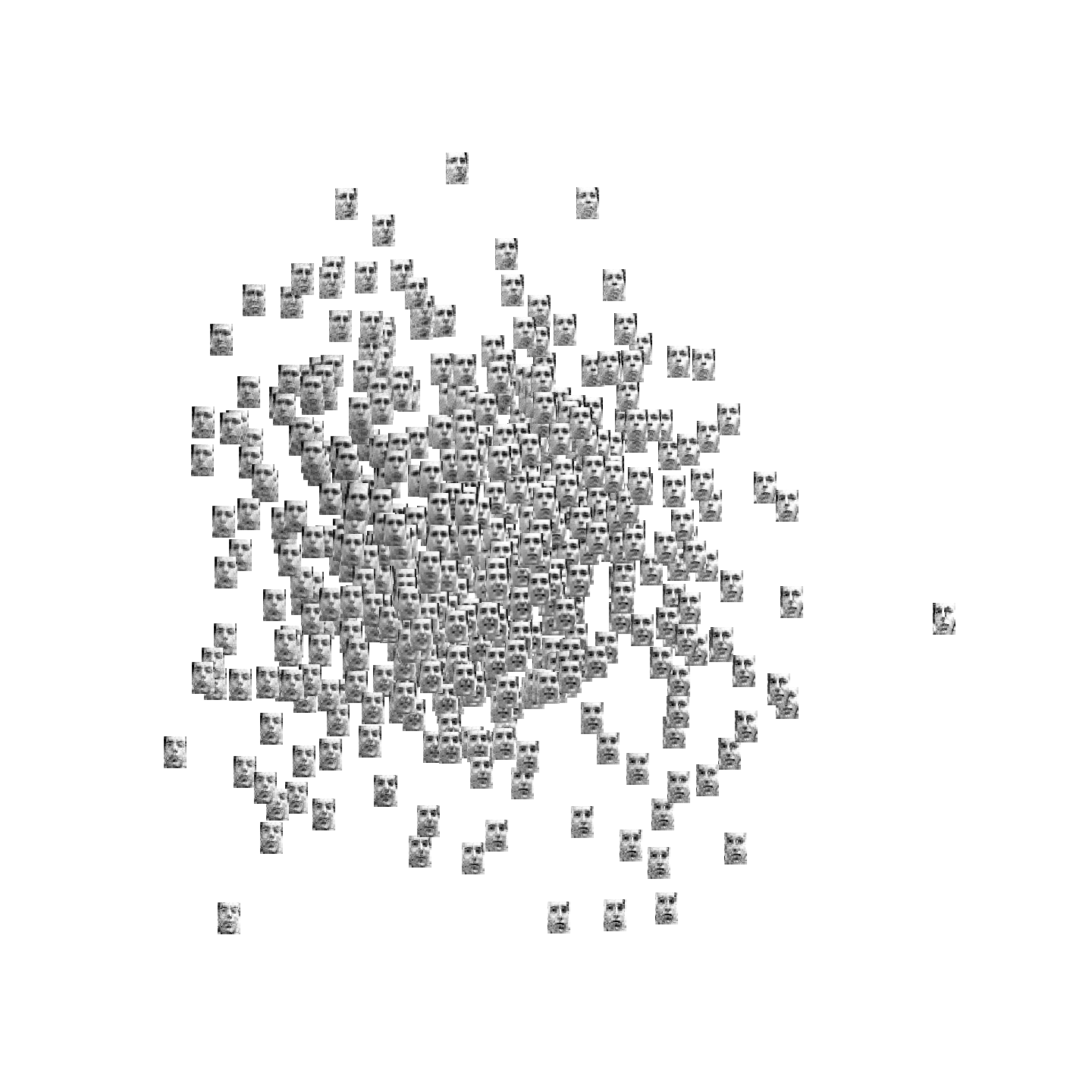}
    }
    \subfigure[SVAE]{
    \includegraphics[width=0.45 \textwidth]{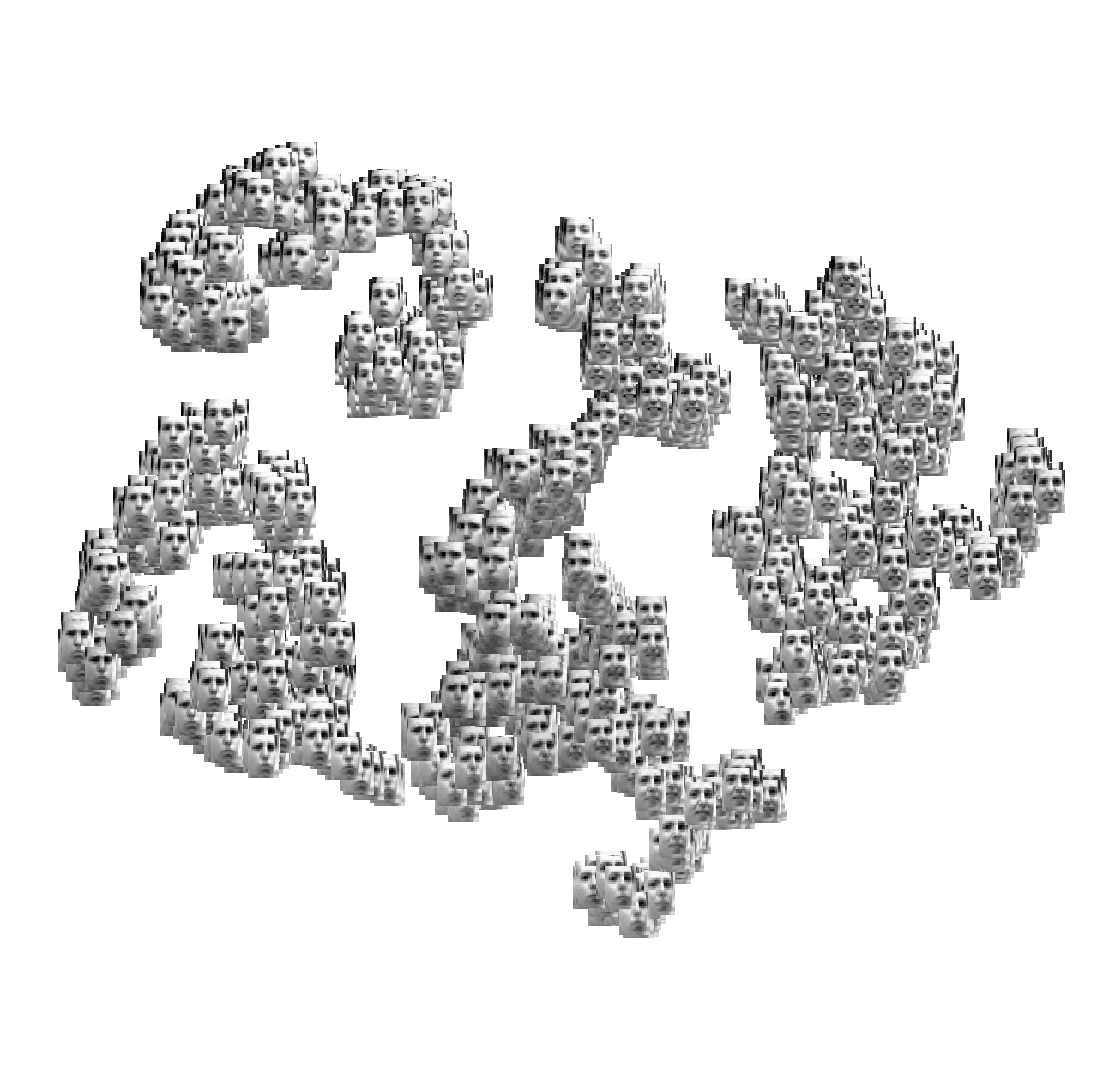}
    }
    \subfigure[VDAE]{
    \includegraphics[width=0.45 \textwidth]{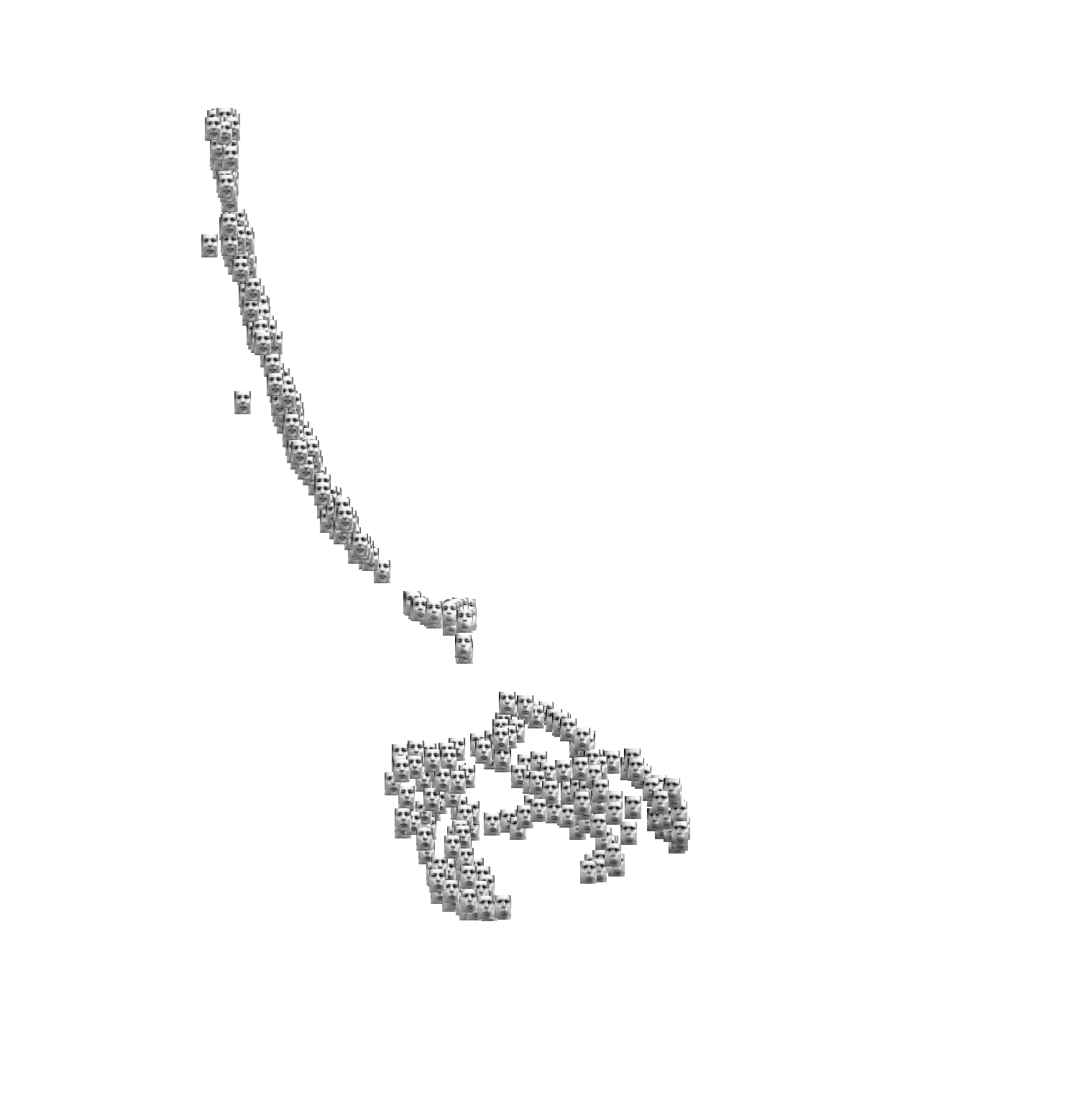}
    }
    
\caption{A tSNE plot of generated images from Frey data set. While the images from the VAE and GAN are compelling, they do not capture the geometric structure of the Frey faces dataset. This structure is much more apparent in the images generated by SVAE and VDAE. In particular, the VDAE has captured a linear structure in the data, which reflects the fact that the dataset was created from a video.} \label{fig:frey}
\end{figure}

\begin{figure}[t]
    \centering
    %\textcolor{red}{Add figures here}
%    \subfigure[VDAE]{
%     \includegraphics[width=1.8cm]{img/bulldog/buldog_gen3.png}
%     }
%     \subfigure[SVAE]{
%     \includegraphics[width=1.8cm]{img/bulldog/buldog_spherical_vae_latent.png}
%     }
    \subfigure[VAE]{
    \includegraphics[width=0.45 \textwidth]{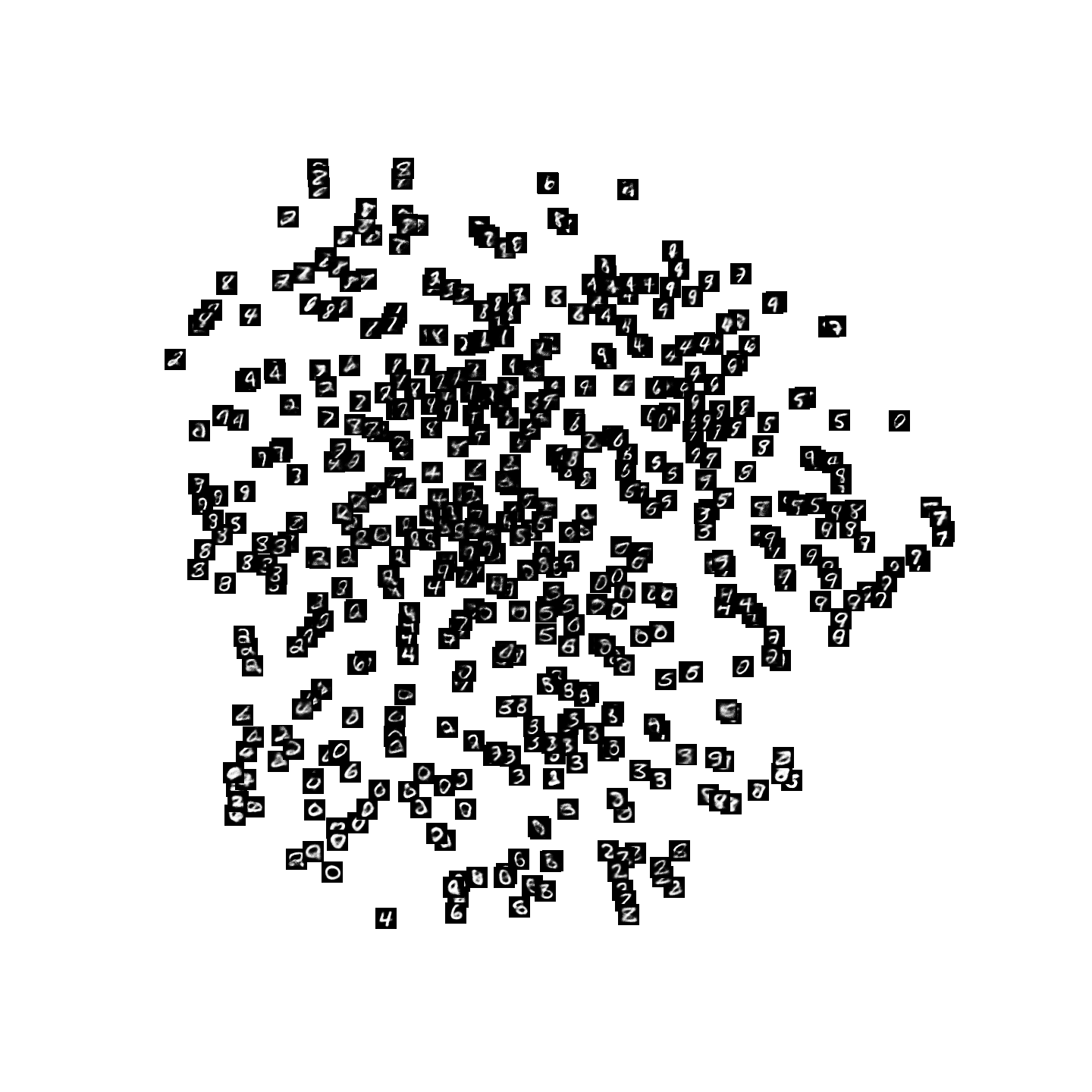}
    }
    \subfigure[GAN]{
    \includegraphics[width=0.45 \textwidth]{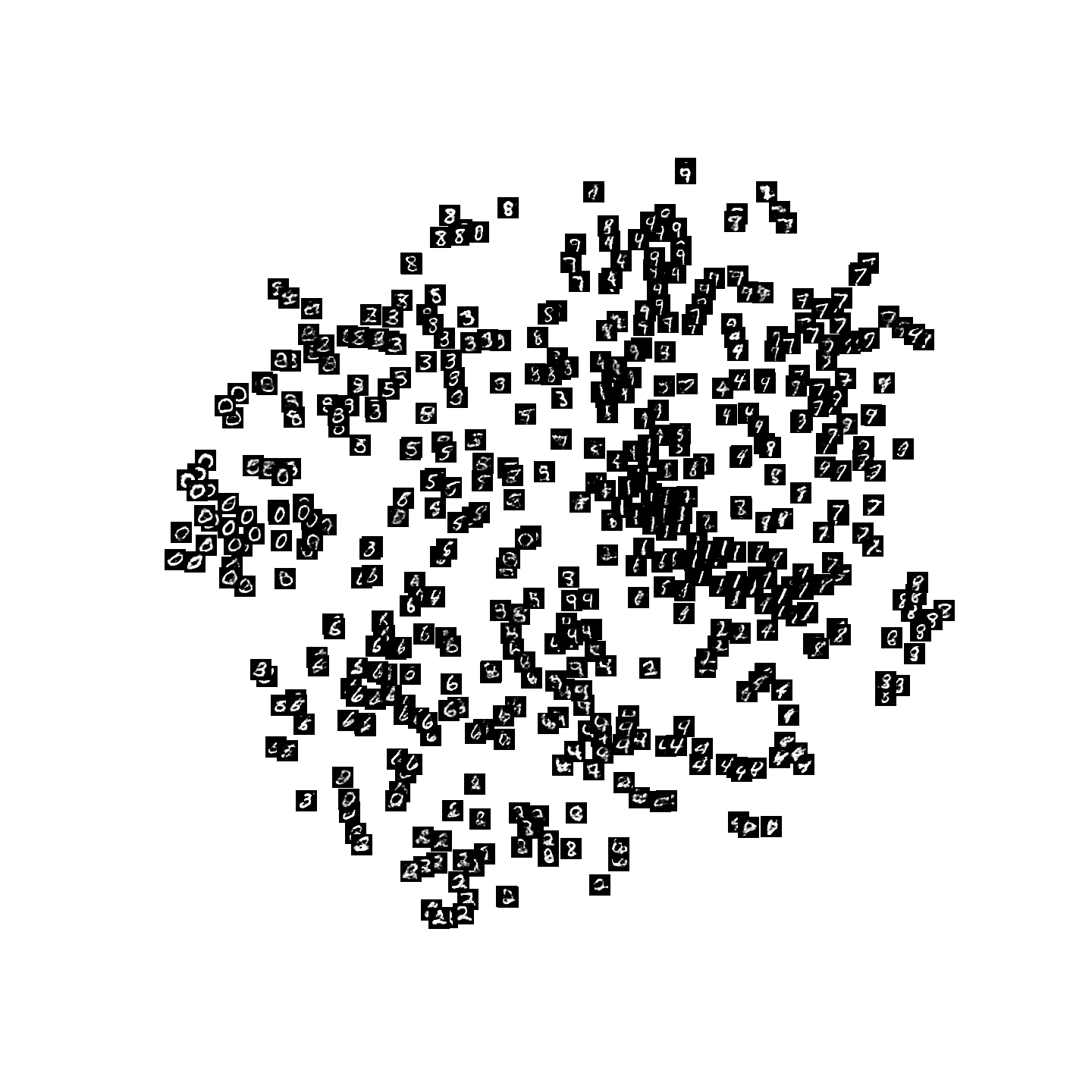}
    }
    \subfigure[SVAE]{
    \includegraphics[width=0.45 \textwidth]{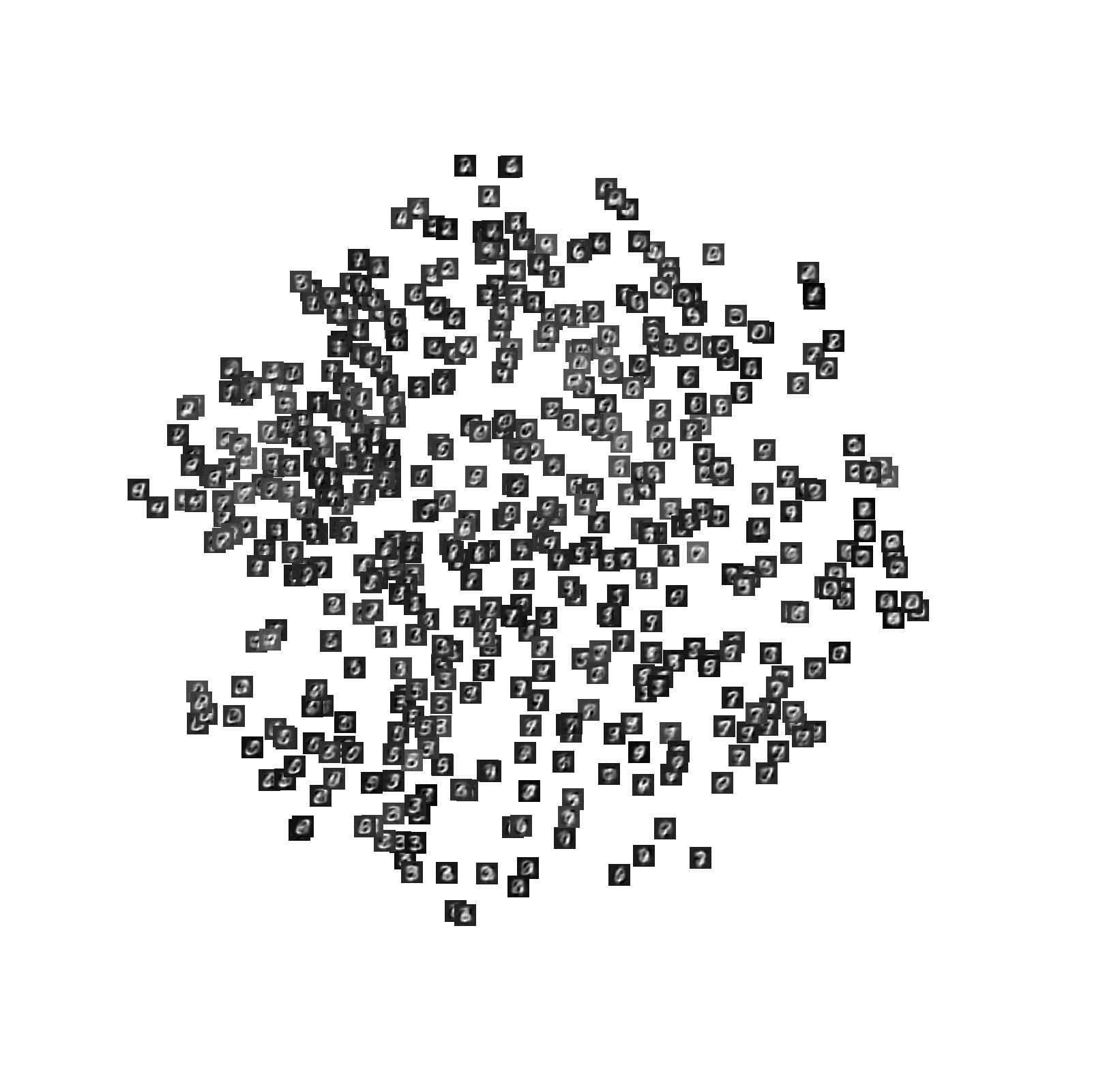}
    }
    \subfigure[VDAE]{
    \includegraphics[width=0.45 \textwidth]{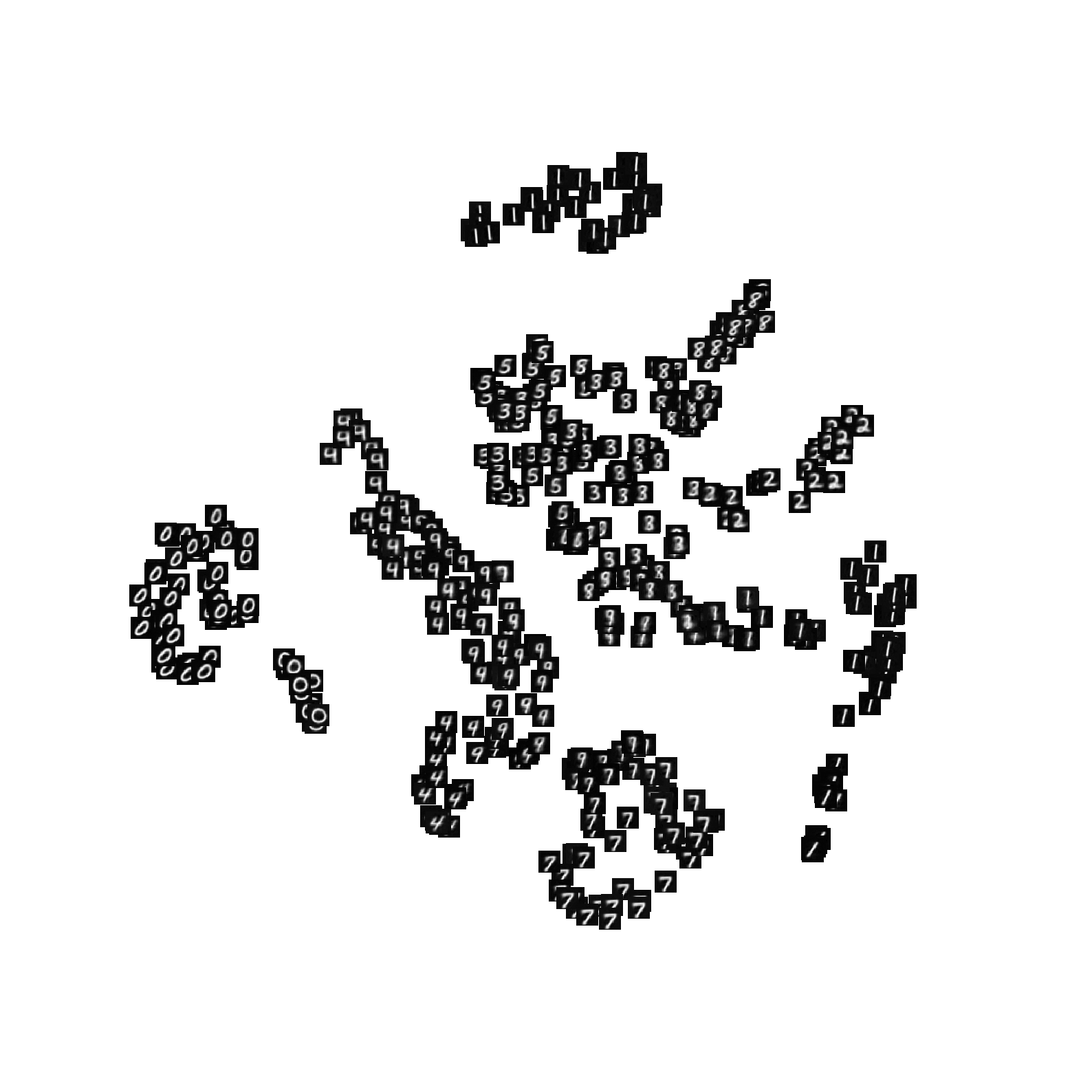}
    }
\caption{A tSNE plot of generated images from MNIST data set. Like with Fig. \ref{fig:frey} (Frey faces), the images generated by VAE, GAN, and SVAE have a unimodal distribution that does not capture the clustered structure of the MNIST dataset. VDAE, on the other hand, organizes the digits into clear clusters, and does not generate from regions where there is low support in the training distribution.} \label{fig:mnist}
\end{figure}

\begin{figure}[t]
    \centering
  \subfigure[VDAE]{
     \includegraphics[width=5.5cm]{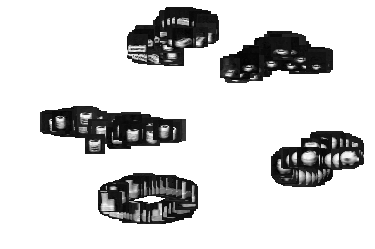}
     }
    \subfigure[SVAE]{
    \includegraphics[width=5.5cm]{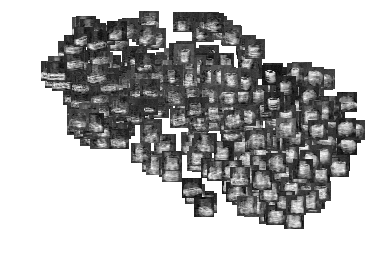}
    }
    \\
    \subfigure[$\beta$-VAE]{
    \includegraphics[width=5.5cm]{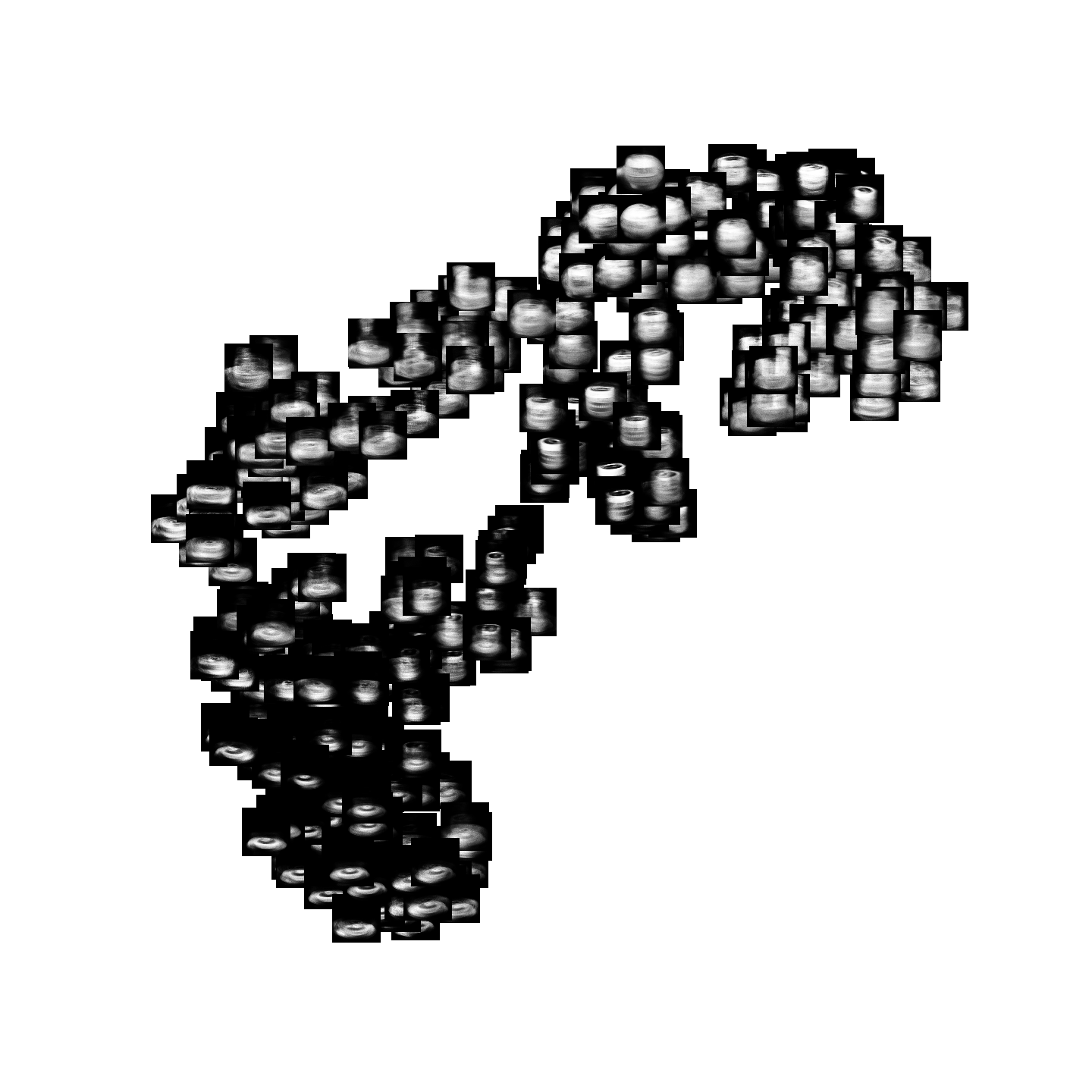}
    }
    \subfigure[WGAN]{
    \includegraphics[width=5.5cm]{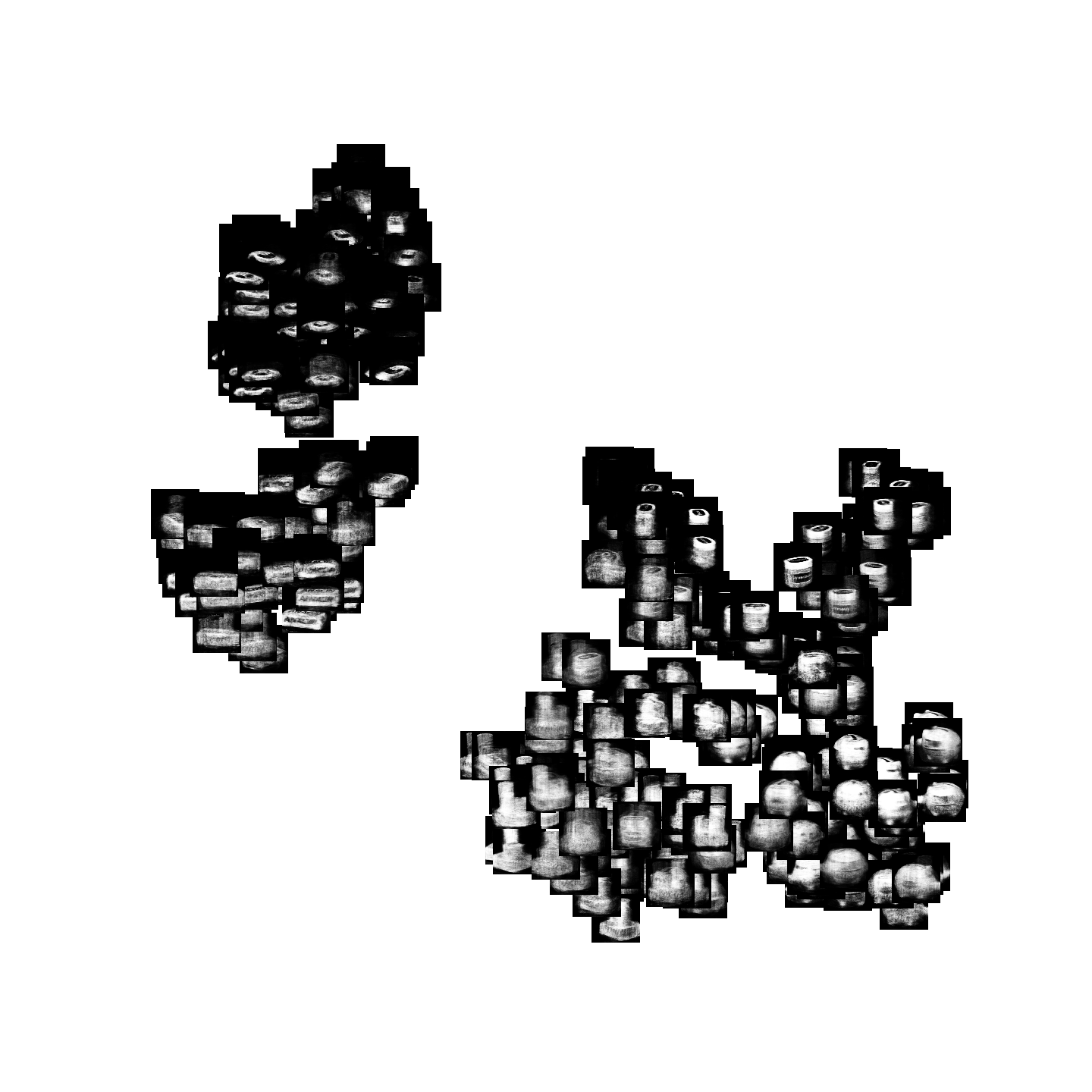}
    }
\caption{A tSNE embedding of $360$ generated images from COIL-20 data set.} \label{fig:coil20}
\end{figure}

\subsection{Experimental Architectures}
\label{sec:architectures}
For the circle, torus, Stanford bunny, Frey faces \footnote{https://cs.nyu.edu/~roweis/data.html}, and the 5x5 spherical density datasets, we used a single 500-unit hidden layer network for all models used in the paper (i.e. decoder, encoder, generator, discriminator, for the VAE, Wasserstein GAN, hyperspherical VAE, and our method).

As higher dimensional datasets, we used a slightly larger architecture for the MNIST, COIL-20, and rotating bulldog datasets: two hidden-layer decoder/generators of width 1024 and 2048, and two hidden-layer encoder/discriminators of width 2048 and 1024. All activations are still ReLU.

\end{document}